\documentclass{article}

\usepackage[accepted]{icml2024/icml2024}
\usepackage[T1]{fontenc}

\usepackage{microtype}
\usepackage{graphicx}
\usepackage{subcaption}
\usepackage{float}
\usepackage{booktabs} %

\usepackage[normalem]{ulem}
\usepackage{array}

\usepackage{hyperref}

\usepackage{wrapfig}

\usepackage{amsmath}
\usepackage{amssymb}
\usepackage{mathtools}
\usepackage{amsthm}

\usepackage[capitalize,noabbrev]{cleveref}

\usepackage{mathrsfs}
\usepackage{caption}
\newcommand{\subfiglabel}[1]{%
    \captionsetup{labelformat=empty}%
    \caption{Figure \thefigure\alph{subfigure}: #1}%
}
\usepackage{graphicx}
\hypersetup{
    colorlinks=true,
    linkcolor={blue!50!black},
    citecolor={cyan!50!black},
}
\usepackage{url}
\usepackage{nicefrac}

\usepackage{thm-restate}

\usepackage[export]{adjustbox}

\theoremstyle{plain}
\newtheorem{theorem}{Theorem}[section]

\newtheorem{lemma}[theorem]{Lemma}

\theoremstyle{definition}
\newtheorem{definition}[theorem]{Definition}

\theoremstyle{remark}
\newtheorem{remark}[theorem]{Remark}

\usepackage{amsmath,amssymb,amsfonts,amsthm}
\usepackage{ifthen}
\usepackage{mathrsfs}
\usepackage{mathalpha}

\def\eqref#1{equation~\ref{#1}}
\def\Eqref#1{Equation~\ref{#1}}

\DeclareMathAlphabet{\mathsfit}{\encodingdefault}{\sfdefault}{m}{sl}
\SetMathAlphabet{\mathsfit}{bold}{\encodingdefault}{\sfdefault}{bx}{n}

\newcommand{\R}{\mathbb{R}}

\newcommand{\expectation}[2]{\mathbb{E}_{#1}\left[#2\right]}

\newcommand{\Conditional}[2]{\left. #1\ \right\rvert\ #2}

\newcommand{\conditionalexpectation}[3]{\expectation{#1}{\Conditional{#2}{#3}}}

\newcommand{\characteristic}[1]{\chi_{#1}}

\newcommand{\law}[1]{\mathrm{Law}(#1)}

\newcommand{\pushforward}[2]{#1_\sharp #2}
\newcommand{\boundedfn}[1]{\mathrm{B}(#1)}

\newcommand{\probset}[1]{\mathcal{P}\left(#1\right)}

\newcommand\usequantumnotation{true}
\newcommand{\measurement}[3]{
  \ifthenelse{ \equal{\usequantumnotation}{true} }{
    \left\langle #1\left\lvert #2\right\rvert #3\right\rangle
  }{
    #1^\top #2 #3
  }
}

\DeclareMathAlphabet{\mathmybb}{U}{bbold}{m}{n}

\NewDocumentCommand{\indicator}{m}{%
  \mathop{}\!\boldsymbol{\mathmybb{1}}
  \mathord{
    \{
  }
  #1
  \mathord{
    \}
  }
}

\DeclareMathOperator{\identity}{\mathsf{id}}

\DeclareMathOperator{\eqlaw}{\overset{\mathcal{L}}{=}}

\usepackage{xspace}

\newcommand{\AbsObject}{Distributional successor measure\xspace}
\newcommand{\absobject}{distributional successor measure\xspace}
\newcommand{\AbsObj}{Distributional SM\xspace}
\newcommand{\absobj}{distributional SM\xspace}

\newcommand{\AbsObjs}{Distributional SMs\xspace}
\newcommand{\absobjs}{distributional SMs\xspace}
\newcommand{\AbsObjAcronym}{DSM\xspace}

\newcommand{\algobject}{$\delta$-model\xspace}

\newcommand{\algobj}{\algobject}
\newcommand{\AlgObjects}{$\delta$-models\xspace}
\newcommand{\algobjects}{$\delta$-models\xspace}
\newcommand{\AlgObjs}{\AlgObjects}
\newcommand{\algobjs}{\algobjects}

\renewcommand{\probset}[1]{\mathscr{P}(#1)}
\newcommand{\randret}[2][]{G_{#1}^{#2}}

\newcommand{\randocc}[1][\pi]{M^{#1}}
\newcommand{\dsr}[1][\pi]{\daleth^{#1}}

\newcommand{\occspace}[1][\mathcal{X}]{\probset{#1}}

\newcommand{\mmd}[1][\kappa]{\ensuremath{\mathrm{MMD}_{#1}}}

\newcommand{\mmdsamples}[1][\kappa]{\ensuremath{\widehat{\mathrm{MMD}}_{#1}}}
\newcommand{\innerkernel}{\ensuremath{\kappa}}
\newcommand{\outerkernel}[2][\innerkernel]{\ensuremath{k_{#1}^{#2}}}

\newcommand{\bootpush}[2][\gamma]{\ensuremath{\mathrm{b}_{#2,#1}}}

\newcommand{\defeq}{\vcentcolon=}

\usepackage{bbm}

\newcommand{\statetransitionkernel}{\ensuremath{p^\pi}}

\begin{document}

\twocolumn[
\icmltitle{A Distributional Analogue to the Successor Representation}

\icmlsetsymbol{equal}{*}

\begin{icmlauthorlist}
\icmlauthor{Harley Wiltzer}{equal,mcgill,mila}
\icmlauthor{Jesse Farebrother}{equal,mcgill,mila,gdm}
\icmlauthor{Arthur Gretton}{gdm,ucl}
\icmlauthor{Yunhao Tang}{gdm}
\icmlauthor{Andr\'e Barreto}{gdm}
\icmlauthor{Will Dabney}{gdm}
\icmlauthor{Marc G. Bellemare}{mcgill,mila,cifar}
\icmlauthor{Mark Rowland}{gdm}
\end{icmlauthorlist}

\icmlaffiliation{mcgill}{McGill University}
\icmlaffiliation{mila}{Mila - Qu\'ebec AI Institute}
\icmlaffiliation{gdm}{Google DeepMind}
\icmlaffiliation{ucl}{Gatsby Unit, University College London}
\icmlaffiliation{cifar}{CIFAR AI Chair}

\icmlcorrespondingauthor{Harley Wiltzer}{harley.wiltzer@mail.mcgill.ca}
\icmlcorrespondingauthor{Jesse Farebrother}{jfarebro@cs.mcgill.ca}

\icmlkeywords{Machine Learning, ICML}

\vskip 0.3in
]

\printAffiliationsAndNotice{\icmlEqualContribution} %

\begin{abstract}
This paper contributes a new approach for distributional reinforcement learning which 
elucidates a clean separation of transition structure and reward in the learning process. Analogous to how the successor representation (SR) describes the expected consequences of behaving according to a given policy, our distributional successor measure (SM) describes the distributional consequences of this behaviour.
We formulate the distributional SM as a distribution over distributions and provide theory connecting it with distributional and model-based reinforcement learning.
Moreover, we propose an algorithm that learns the \absobj from data by minimizing a two-level maximum mean discrepancy. Key to our method are a number of algorithmic techniques that are independently valuable for learning generative models of state.
As an illustration of the usefulness of the \absobj, we show that it
enables zero-shot risk-sensitive policy evaluation in a way that was not previously possible.
\end{abstract}

\section{Introduction}\label{sec:introduction}

Distributional reinforcement learning \citep{morimura2010nonparametric, Bellemare2017ADP, bdr2022} is an approach to reinforcement learning (RL) that focuses on learning the entire probability distribution of an agent's return, not just its expected value. Distributional RL has been shown to improve deep RL agent performance \citep{yang2019fully,nguyen2020distributional}, and provides a flexible approach to risk-aware decision-making \citep{dabney2018implicit,zhang2021safe,fawzi2022discovering}.
A notable drawback of existing approaches to distributional RL is that rewards must be 
available at training time in order to predict the return distribution. 
For example, if we wish to evaluate a trained policy on a new task with regard
to various performance criteria,
these 
predictions of the return distributions must be trained from scratch.
This paper contributes a method that overcomes this drawback, allowing for
\emph{zero-shot evaluation} of novel reward functions without requiring further 
learning.

In the case of predicting just the \emph{expected} return, such zero-shot evaluation is made possible by learning the successor representation \citep[SR;][]{dayan1993improving}. This approach has recently been extended to continuous state spaces \citep{blier2021learning, blier2022some}, and the introduction of a variety of density modelling and generative modelling techniques mean that such zero-shot transfer is now possible at scale \citep{janner2020gamma,touati2021learning,touati2023does}.

This paper extends the idea of the successor representation to distributional RL, by defining the \emph{distributional successor measure} (DSM). We show that the DSM is a reward-agnostic object that can by combined with any deterministic reward function to obtain the corresponding distribution of returns, extending zero-shot transfer to the entire distribution of returns.
Our primary algorithmic contribution is the $\delta$-model, a tractable approximation to the distributional successor measure based on ensembles of diverse generative models, along with practical implementation techniques that are crucial for success.
We exhibit the power of \algobjs by demonstrating their unique ability
generalize across tasks and \emph{risk-sensitive} criteria
without necessitating any further data collection or training, which could be expensive or dangerous -- a feat that no other method can accomplish.

\section{Background}\label{sec:background}

In the sequel, 
$\law{X}$ denotes the probability measure governing a random variable $X$, and $X\eqlaw Y$ (read \emph{equal in distribution})
is written to indicate that $\law{X}=\law{Y}$. The notation $\probset{A}$ defines the space of probability
measures over a set $A$. We also write $(X,Y)\sim\mu\otimes\nu$ to refer to the pair of independent samples $X\sim\mu,\, Y\sim\nu$.

We consider a Markov decision process (MDP) with state space $\mathcal{X}$, finite action space 
$\mathcal{A}$, transition kernel $p : \mathcal{X}\times\mathcal{A} \to \probset{\mathcal{X}}$, 
bounded and measurable reward function $r : \mathcal{X}\to\R$, and discount factor $\gamma \in [0,1)$. We
assume henceforth that $\mathcal{X}$ is a complete and separable metric space, which allows for finite state spaces, as well as many continuous state spaces of interest.
Given a policy $\pi : \mathcal{X} \to 
\probset{\mathcal{A}}$ and initial state $x_0 \in \mathcal{X}$ drawn from $\mu_0\in\probset{\mathcal{X}}$, an agent generates a random 
trajectory $(X_t, A_t, R_t)_{t = 0}^\infty$ of states, actions, and rewards, with distributions 
specified by $X_0 = x$, $A_t \sim \pi(\cdot|X_t)$, $R_t = r(X_t)$, and $X_{t+1} \sim 
p(\cdot|X_t, A_t)$ for all $t \geq 0$. For a fixed policy $\pi$, we will denote the transition 
kernel governing state evolution by $\statetransitionkernel$, where
$\statetransitionkernel(\cdot\mid x) = \sum_{a\in \mathcal{A}}p(\cdot\mid x, a)\pi(a\mid x)$.

The (random) return summarises the performance of the agent along its trajectory, and
for each possible initial state
$X_0=x\in\mathcal{X}$, it is defined as $\randret[r]{\pi}(x)\defeq\sum_{t=0}^\infty
\gamma^tr(X_t)$.
When there is no
ambiguity about the reward function, we will write $\randret{\pi}$ in
place of $\randret[r]{\pi}$.
For a given policy $\pi$, the problem of \emph{policy evaluation} is to find the expected return for each initial state. Mathematically, this can be expressed as learning the function $V^\pi_r : \mathcal{X} \rightarrow \mathbb{R}$, defined by
$V_r^\pi(x) \defeq \expectation{}{\randret[r]{\pi}(x)}$,
this describes the quality of $\pi$ in its own right, and may also be used to obtain 
\emph{improved} policies, for example by acting greedily  
\citep{puterman2014markov}.

\subsection{Successor Measure}

The \emph{normalized successor measure} $\Psi^\pi:\mathcal{X}\to\probset{\mathcal{X}}$ associated with a policy $\pi$
is defined by
\begin{equation}\label{eq:sr}
\Psi^\pi(S\mid x) \defeq \sum_{t=0}^\infty (1-\gamma) \gamma^t\Pr(X_t\in S\mid X_0=x) \, ,
\end{equation}
for any (measurable) set $S\subseteq\mathcal{X}$\footnote{This covers discounted occupancies over Polish state spaces, including compact Euclidean space.} and initial state $x \in \mathcal{X}$.
In the literature, $\Psi^\pi(\cdot\mid x)$ is often referred to
as the (discounted) state occupancy measure
\footnote{Here, occupancy measure is conditional on a source state; the usual occupancy is given by $\mu^\pi(S) = \int_{\mathcal{X}}\Psi^\pi(S\mid x)\,\mu_0(\mathrm{d}x)$.}
The object $\Psi^\pi$ described above is a normalised version of the \emph{successor representation} \citep[SR; ][]{dayan1993improving} in the tabular case and the \emph{successor measure} \citep[SM; ][]{blier2021learning,touati2021learning} for continuous state~spaces.
\citet{blier2021learning} shows that, without the $(1-\gamma)$ factor in \Eqref{eq:sr}, $\Psi^\pi(\cdot\mid x)$ is a measure
for each $x\in\mathcal{X}$ with total mass $(1-\gamma)^{-1}$. We include the $(1-\gamma)$ normalizing factor so that $\Psi^\pi(\cdot\mid x)$
is in fact a probability distribution -- this allows for one to sample from the successor measure, as in the work of \citet{janner2020gamma}.
Intuitively, $\Psi^\pi(S\mid x)$ describes the proportion of time spent in the region $S\subseteq\mathcal{X}$, in expectation, weighted by the discount factor according to the time of visitation.

Since for each $x \in \mathcal{X}$, $\Psi^\pi(\cdot|x)$ is a probability distribution over states, we can compute expectations under this distribution.
Notably, the reward function $r$, successor measure $\Psi^\pi$, and value function~$V^\pi$
satisfy the following identity,
\begin{align}\label{eq:zero-shot-eval}
    V^\pi_r(x) = (1-\gamma)^{-1} \mathbb{E}_{X' \sim \Psi^\pi(\cdot|x)}[r(X')] \, ,
\end{align}
as leveraged in the recent work of \citet{janner2020gamma} and \citet{blier2021learning}.
In words, the value function can be expressed as an expectation of the reward, with respect to the successor measure $\Psi^\pi(\cdot|x)$; this expression cleanly factorises the value function into components comprising transition information and reward information, and generalises the result in the tabular case by \citet{dayan1993improving}. A central consequence is that learning $\Psi^\pi$ allows for the evaluation of $\pi$ on unseen reward functions, without further learning; this is known as \emph{zero-shot policy evaluation}.

\subsection{Distributional Policy Evaluation}

In distributional reinforcement learning \citep{morimura2010nonparametric,Bellemare2017ADP,bdr2022}, the problem of \emph{distributional policy evaluation} is concerned with finding not just the expectation of the random return, but its full probability distribution. Analogous to our description of policy evaluation above, this can be mathematically expressed as aiming to learn the return-distribution function $\eta^\pi_r : \mathcal{X} \rightarrow \mathscr{P}(\mathbb{R})$, with $\eta^\pi_r(x)$ equal to the distribution of $\randret[r]{\pi}(x)$.

An added complication in the distributional setting is that the return distributions are infinite-dimensional, in contrast with the scalar mean returns learned in classical reinforcement learning. This requires careful consideration of how probability distributions will be represented algorithmically, with common choices including categorical \citep{Bellemare2017ADP} and quantile \citep{Dabney2018DistributionalRL} approaches; see \citet[][Chapter~5]{bdr2022} for a summary.

\section{The \AbsObj}\label{sec:dsr}
\begin{figure*}
\centering
\begin{minipage}{0.2\linewidth}
\vspace{-0.8cm}
\includegraphics[width=\linewidth]{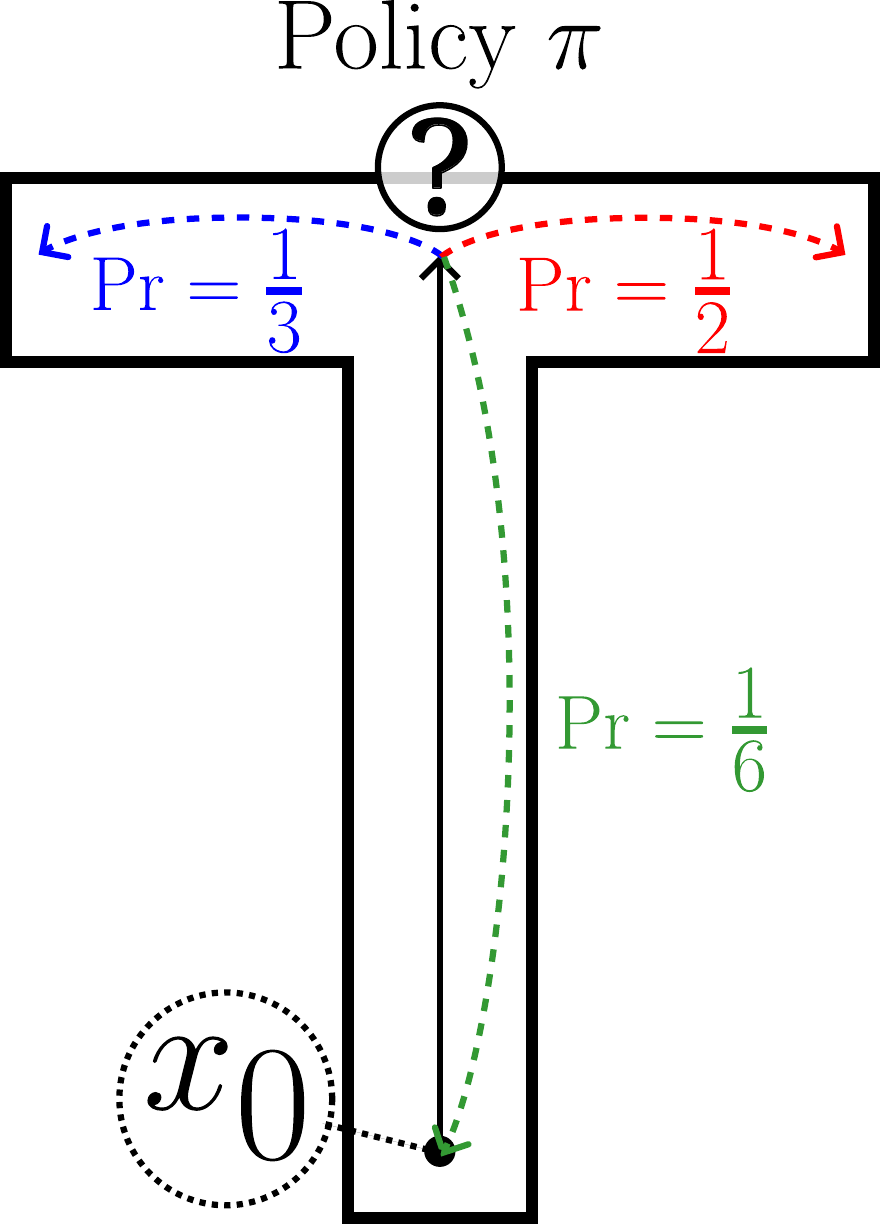}
\end{minipage}
\hfill
\begin{minipage}{0.2\linewidth}
\includegraphics[width=\linewidth, trim={5cm 0 0 0}, clip]{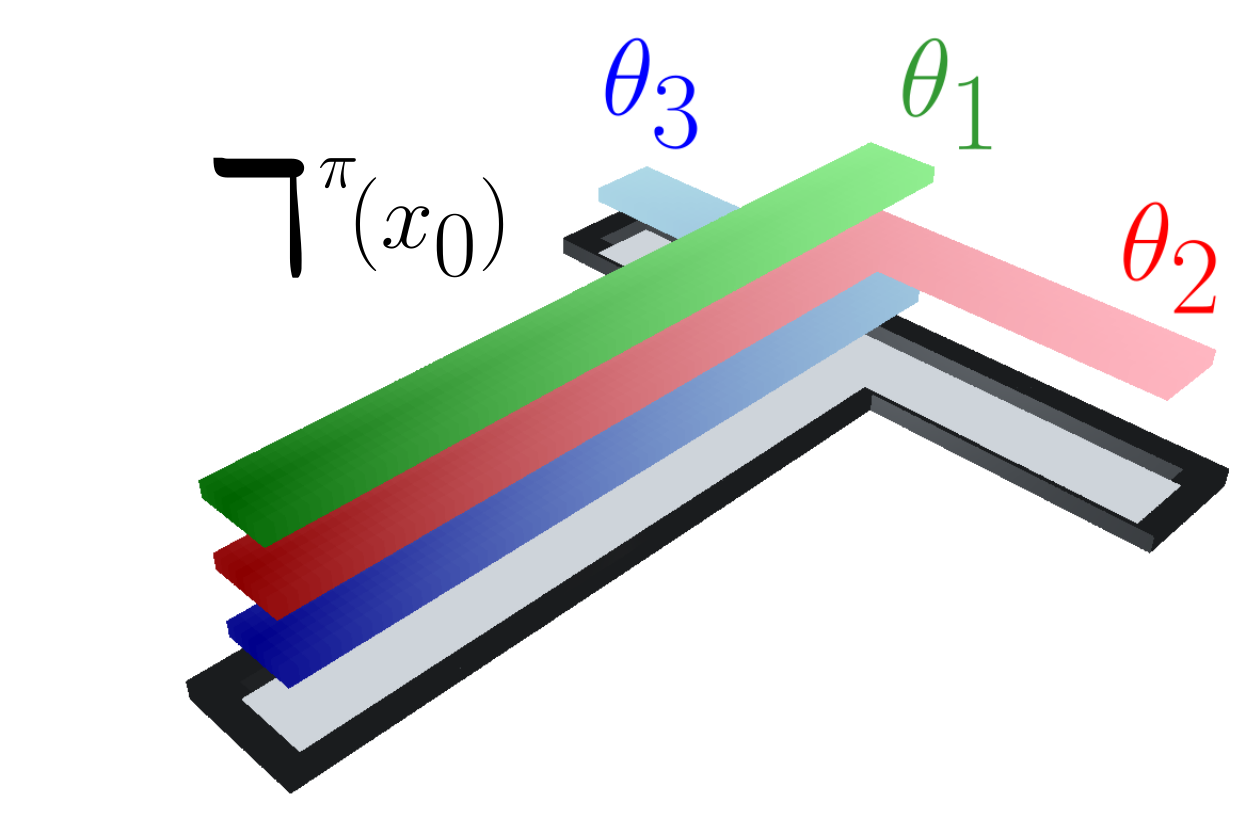}
\includegraphics[width=\linewidth, trim={5cm 0 0 0}, clip]{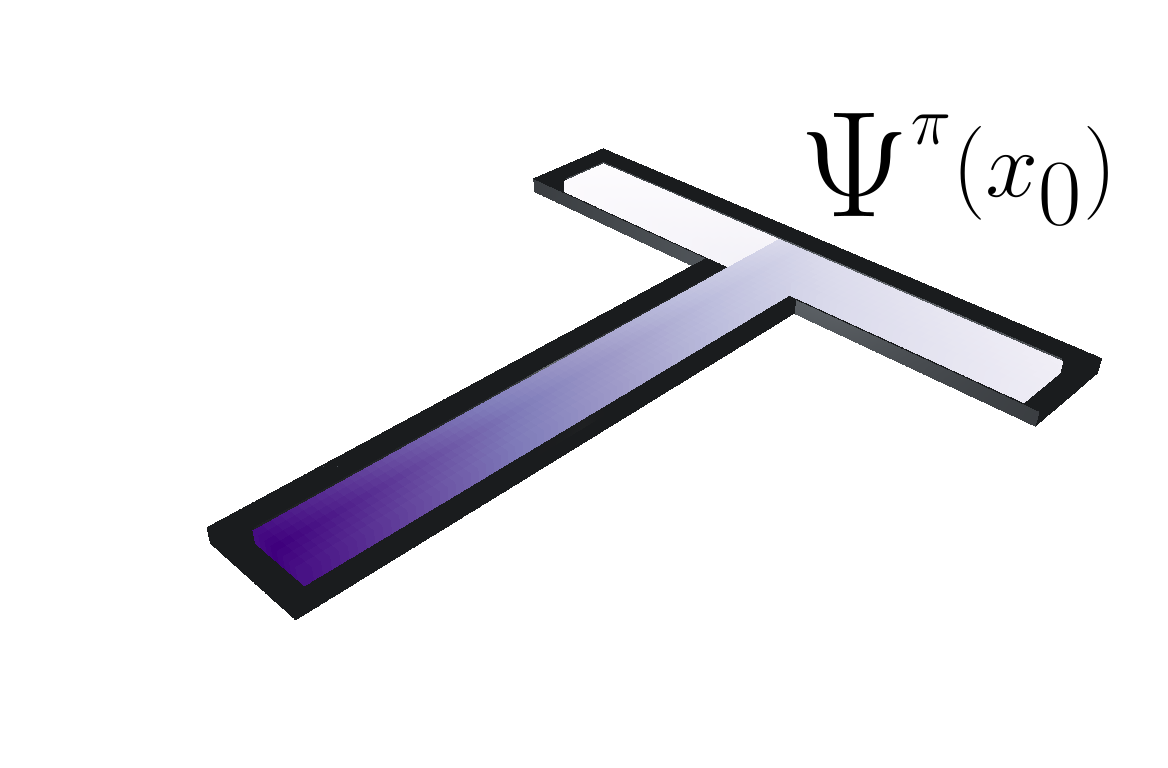}
\end{minipage}
\hfill
\begin{minipage}{0.5\linewidth}
\includegraphics[width=\linewidth]{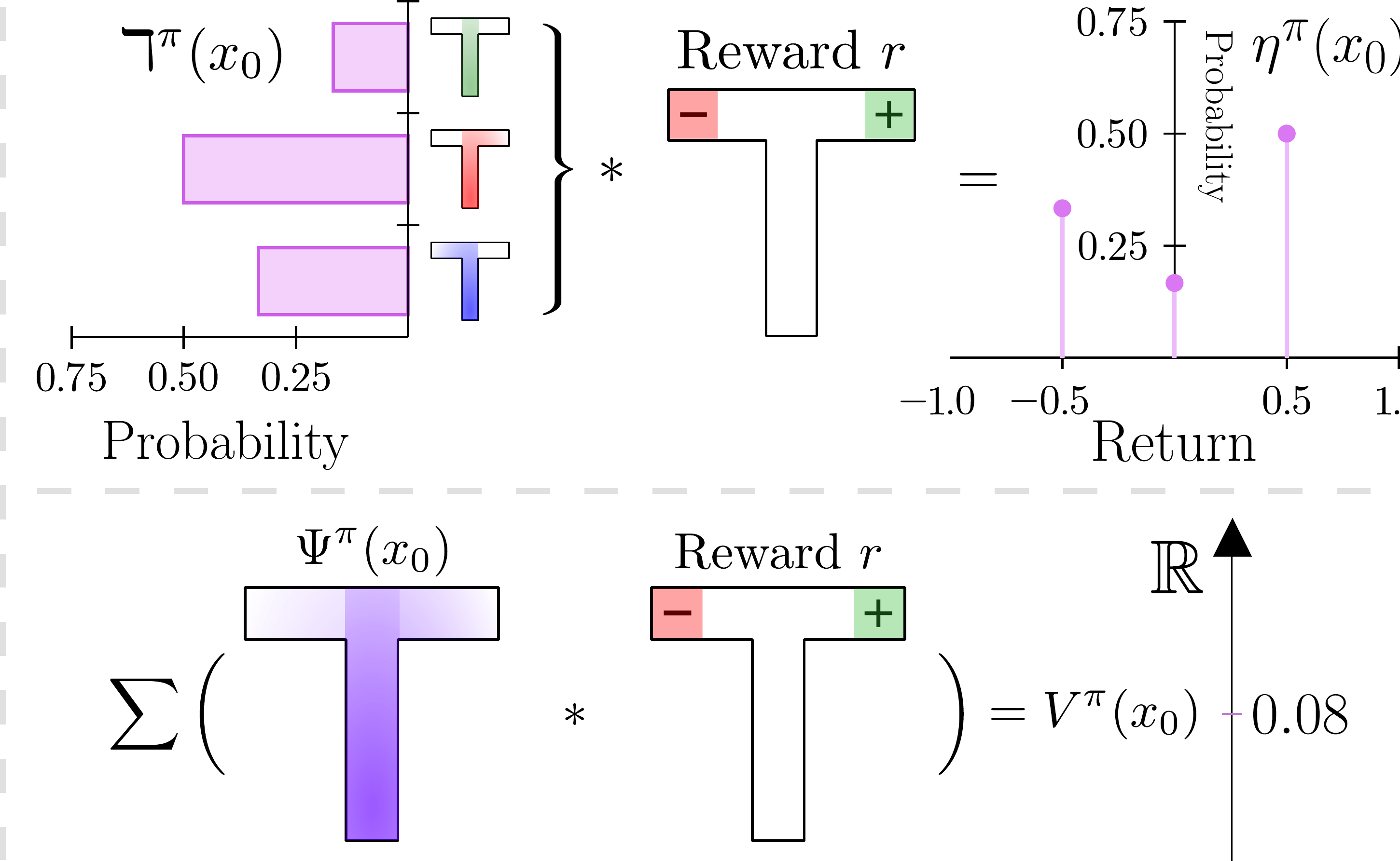}
\end{minipage}
\caption{
Illustration of the standard and distributional successor measure (SM) in a T-Maze 
MDP, for a policy that moves to the fork and goes backwards, right, or left, 
with probabilities $\frac{1}{6},\frac{1}{2},\frac{1}{3}$. \textbf{Left}: The \absobj 
$\dsr$ (top) consisting of atoms $\theta_1, \theta_2, \theta_3$ depicting the
occupancy measures (probability distributions) corresponding to the distinct behaviors exhibited by the policy, 
and the SM $\Psi^\pi$ (bottom)
$\Psi^\pi = \frac{\theta_1}{6} + \frac{\theta_2}{2} +  \frac{\theta_3}{3}$.
\textbf{Right}: Zero-shot distributional policy evaluation (top) with $\dsr$
and zero-shot policy evaluation (bottom) with $\Psi^\pi$.
}
\label{fig:dsr_and_sr}
\end{figure*}

One of the core contributions of this paper is to introduce a mathematical object that plays the role of the successor measure in distributional reinforcement learning.
Analogous to how distributional RL models the distribution of the return, we  study the \emph{distribution} over future state occupancies.

\subsection{Random Occupancy Measures}

To begin, we contribute a new form for the normalised successor measure (SM), which shows that it can be written as an expectation of the discounted visitation distribution for the \emph{random} state sequence $(X_t)_{t \geq 0}$ generated by $\pi$:
\begin{equation*}
    \Psi^\pi(S\mid x) = \conditionalexpectation{\pi}{ \sum_{k=0}^\infty (1-\gamma) \gamma^k \delta_{X_k}(S)}{X_0 = x}
 \end{equation*}
for all measurable $S\subset\mathcal{X}$. Here, $\delta_{X_k}$ is the probability distribution over $\mathcal{X}$ that puts all its mass on $X_k$, so that $\delta_{X_k}(S) = \indicator{ X_k \in S}$.
We obtain a distributional version of this object by ``removing the expectation".

\begin{definition}[Random occupancy measure]\label{def:randocc}
    For a given policy $\pi$, let $(X_t)_{t=0}^\infty$ be a random sequence of states generated by interacting with the environment via $\pi$. The associated \emph{random discounted state-occupancy measure} $M^\pi$ assigns to each initial state $x \in \mathcal{X}$ a random probability distribution $M^\pi(\cdot\mid x)$ according to
    \begin{equation}
        \label{eq:random-occupancy-measure}
        \randocc(S\mid x) \defeq
        \sum_{k=0}^\infty (1-\gamma) \gamma^k \delta_{X_k}(S),\  X_0 = x \, .%
    \end{equation}%
\end{definition}%
It is worth pausing to consider the nature of the object we have just defined. For each $x \in \mathcal{X}$, $M^\pi(\cdot\mid x)$ is a random variable, and each realisation of $M^\pi(\cdot\mid x)$ is a probability distribution over $\mathcal{X}$. So, for any measurable $Y\subset\mathcal{X}$, $M^\pi(Y|x)$ is also a random variable, which gives the discounted proportion of time spent in $Y$ across different possible sampled trajectories. Thus, the distribution of $M^\pi(\cdot\mid x)$ is a distribution \emph{over} probability distributions; see Figure~\ref{fig:dsr_and_sr}.

As described in Section~\ref{sec:background}, an important property of the successor representation is that it is a linear operator that maps reward functions to value functions. The next result shows that $M^\pi$ can be used to map reward functions to random returns; all proofs are given in Appendix \ref{sec:proofs}.

\begin{restatable}{proposition}{randocclinear}\label{prop:randocc:linear-operator}
    Let $M^\pi$ denote a random discounted state-occupancy measure for a given policy $\pi$.
    For any deterministic reward function $r:\mathcal{X}\to\R$, we have
    \begin{equation}
      \label{eq:zero-shot-return-distribution}
      \randret[r]{\pi}(x) \eqlaw (1-\gamma)^{-1} \expectation{X'\sim M^\pi(\cdot\mid x)}{r(X')} \, .
    \end{equation}
    Note that the right-hand side \emph{is} a random variable, since $M^\pi(\cdot\mid x)$ itself is a random distribution.
\end{restatable}

Proposition~\ref{prop:randocc:linear-operator} suggests a novel approach to distributional RL.
To obtain return distributions, one can first learn the distribution of $M^\pi$ (without any information about rewards), and then use \Eqref{eq:zero-shot-return-distribution} to obtain an estimate of the corresponding return distribution. This unlocks an ability that was not previously possible in distributional RL: zero-shot distributional policy evaluation. In particular, one can learn the distribution of the random occupancy measure, and then approximate the return distribution associated with \emph{any} reward function $r$ without requiring further learning (see Figure \ref{fig:dsr_and_sr}). Once the return distribution is obtained, the benefits of distributional RL, such as risk estimation, are immediately available, something not possible using SR in isolation.

\begin{remark}
    Perhaps surprisingly, our assumption of a deterministic reward function made in Proposition~\ref{prop:randocc:linear-operator} is necessary
    for a linear factorization between reward functions and return distributions.
    This is due to the statistical
    dependence between random rewards observed along trajectories and random trajectories
    themselves. 
    We explore this in more depth in Appendix~\ref{sec:further-discussion}.
\end{remark}

As in distributional RL, where we distinguish between the random return $G^\pi(x)$ and its distribution $\eta^\pi(x)$, we introduce notation for expressing the distribution of $M^\pi(\cdot\mid x)$.
\begin{definition}[\AbsObject]
    The \emph{\absobject} (\absobj) $\dsr : \mathcal{X} \rightarrow \mathscr{P}(\mathscr{P}(\mathcal{X}))$ is defined by $\daleth^\pi(x) = \law{M^\pi(\cdot\mid x)}$.
\end{definition}
Prior to this work, the only conceivable method for zero-shot distributional
policy evaluation involved learning $\statetransitionkernel$ and estimating return
distributions by sampling rollouts and returns from the learned model.
Indeed, we observe that the distributional SM (and thus the SM itself) is mathematically determined by $\statetransitionkernel$; see Proposition~\ref{prop:SRDeterminesDSR} in Appendix~\ref{sec:additional_results} for a precise statement and proof of this result.
Despite this, recovering the distributional SM or the SM from $\statetransitionkernel$ in large MDPs is intractable, and SM-like models are known to be more
robust to estimation error for long-horizon prediction in continuous MDPs \citep{janner2020gamma,thakoor2022generalised,touati2023does}.
Crucially, unlike an approach to zero-shot distributional evaluation that estimates return distributions by sampling rollouts from a learned $\statetransitionkernel$ and computing MC returns, the \absobj is \emph{not prone to accumulation of model error}, which results in substantially more accurate
estimation, as we show in Section \ref{sec:results}.

Proposition~\ref{prop:randocc:linear-operator} is the core mathematical insight of the paper; we now develop an algorithmic framework for translating these theoretical ideas into concrete implementations.

\subsection{\AbsObj Bellman Equations}

A central result in developing temporal-difference methods for learning $\dsr$ is that $\randocc$ satisfies a \emph{distributional Bellman equation} \citep{morimura2010nonparametric,Bellemare2017ADP}.

\begin{restatable}{proposition}{randoccbellman}\label{prop:randocc:bellman}
Let $M^\pi$ denote the random discounted state-occupancy measure induced by a policy $\pi$. Then $M^\pi$ can be expressed recursively via a distributional Bellman equation: for all measurable $S\subset\mathcal{X}$, and $X'\sim\statetransitionkernel(\cdot\mid x)$,
\begin{equation}
  \label{eq:randocc-bellman-equation}
  \randocc(S\mid x) \eqlaw (1-\gamma)\delta_x(S) + \gamma M^\pi(S\mid X').
\end{equation}
\end{restatable}

This provides a novel reward-agnostic distributional Bellman equation for random occupancy measures. Note that the multi-dimensional reward distributional Bellman equation studied by \citet{freirich2019distributional,zhang2021distributional} can be framed as an instance of \Eqref{eq:randocc-bellman-equation} when $\mathcal{X}$ is finite.

We can also express the \absobj recursively,
\begin{equation}\label{eq:dsr:bellman}
\dsr(x) = \expectation{X'\sim \statetransitionkernel(\cdot\mid x)}{\pushforward{(\bootpush{x})}{\dsr(X')}}
\end{equation}
where $\bootpush{x}:\probset{\mathcal{X}}\to\probset{\mathcal{X}}$ is given by $\bootpush{x}(\mu) 
= (1-\gamma)\delta_x + \gamma\mu$.
The notation $\pushforward{f}{\nu} = \nu\circ f^{-1}$ denotes the \emph{pushforward} of a
measure $\nu$ through a measurable function $f$.
This motivates the following operator on the space of \absobjs having $\dsr$ as a fixed point, which we refer to as the distributional
Bellman operator $\mathcal{T}^\pi:\probset{\occspace}^{\mathcal{X}}\to\probset{\occspace}^{\mathcal{X}}$,
\begin{equation}\label{eq:distributional-bellman-operator:dsm}
    (\mathcal{T}^\pi\daleth)(x) = \expectation{X'\sim\statetransitionkernel(\cdot\mid x)}{(\bootpush{x})_\sharp\daleth(X')}.
\end{equation}
The proceeding statements outline a convergent approach for computing the \absobj by dynamic programming.

\begin{restatable}[Contractivity of $\mathcal{T}^\pi$]{proposition}{dsmcontraction}\label{prop:dsm:contraction}
Let $d$ be a metric on $\mathcal{X}$ such that
$(\mathcal{X}, d)$ is a Polish space, and let $w_d$ denote the Wasserstein distance on $\probset{\mathcal{X}}$ with base distance $d$. If $W : \mathscr{P}(\mathscr{P}(\mathcal{X})) \times \mathscr{P}(\mathscr{P}(\mathcal{X})) \rightarrow \mathbb{R}$ is the Wasserstein distance on $\probset{\probset{\mathcal{X}}}$ with base distance $w_d$, then
\begin{align*}
    \overline{W}(\mathcal{T}^\pi\daleth_1, \mathcal{T}^\pi\daleth_2)\leq\gamma \overline{W}(\daleth_1, \daleth_2),
\end{align*}
where $\overline{W}$ is the ``supremal" $W$ metric given by $\overline{W}(\daleth_1, \daleth_2) = \sup_{x\in\mathcal{X}}W(\daleth_1(x), \daleth_2(x))$.
\end{restatable}

\begin{restatable}[Convergent Dynamic Programming]{corollary}{dsmbanach}\label{cor:dsm:banach}
Under the conditions of Proposition \ref{prop:dsm:contraction}, if the metric space $(\mathcal{X}, d)$ is compact,
then the iterates $(\daleth_k)_{k=0}^\infty$ given by
$\daleth_{k+1} = \mathcal{T}^\pi\daleth_k$
converge in $\overline{W}$ to $\dsr$, for any $\daleth_0\in\probset{\occspace}^{\mathcal{X}}$.
\end{restatable}

The proofs of Proposition \ref{prop:dsm:contraction} and Corollary \ref{cor:dsm:banach} rely on a novel coupling technique on the doubly-infinite-dimensional space $\probset{\occspace}$, which can be found in Appendix \ref{sec:additional_results:ddp}.

\section{Representing and Learning the \AbsObjAcronym}\label{sec:rep-learn-dsr}

The \absobj provides an alternative perspective on distributional reinforcement learning,
and opens up possibilities such as zero-shot distributional policy evaluation, which is not achievable with existing approaches 
to distributional RL. However, to turn these mathematical observations into practical algorithms, we need methods for
efficiently \emph{representing} and \emph{learning} the \absobj.

\subsection{Representation by \algobjs}\label{sec:rep-learn-dsr:disrete}

As in standard distributional RL, we cannot represent $\dsr$ within an algorithm exactly, as it is comprised 
of probability  distributions, which are objects having infinitely-many degrees of freedom. To make matters more complicated 
still, these are distributions not over the real numbers (as in standard distributional RL), but over $\mathscr{P}(\mathcal{X})$,
which may itself have infinitely-many degrees of freedom if $\mathcal{X}$ is infinite.
Thus, a tractable approximate representation is necessary.
We propose the \emph{equally-weighted particle (EWP) representation}, which is inspired
by the quantile representation of return distributions in standard distributional RL 
algorithms \citep{Dabney2018DistributionalRL,nguyen2020distributional}.
Under this representation, the approximation $\dsr[](x)$ of $\dsr(x)$ is represented as a sum of equally-weighted Dirac masses on the set $\probset{\mathcal{X}}$: 
$\dsr[](x) = \frac{1}{m}\sum_{i=1}^m\delta_{\theta_i(x)}$,
with $\theta_i(x) \in \mathscr{P}(\mathcal{X})$. The 
approximation problem now reduces to learning appropriate values $((\theta_i(x))_{i=1}^m 
: x \in \mathcal{X})$ of these Dirac masses.
It is important to note that these Dirac masses inhabit the space $\probset{\mathcal{X}}$ -- each Dirac is located on a
distribution of state. We must
find a set of $m$ state distributions such that the \emph{collection} of the learned distributions is optimal with respect
to a metric on $\probset{\probset{\mathcal{X}}}$.

\begin{figure}
    \centering
    \includegraphics[keepaspectratio,width=\linewidth]{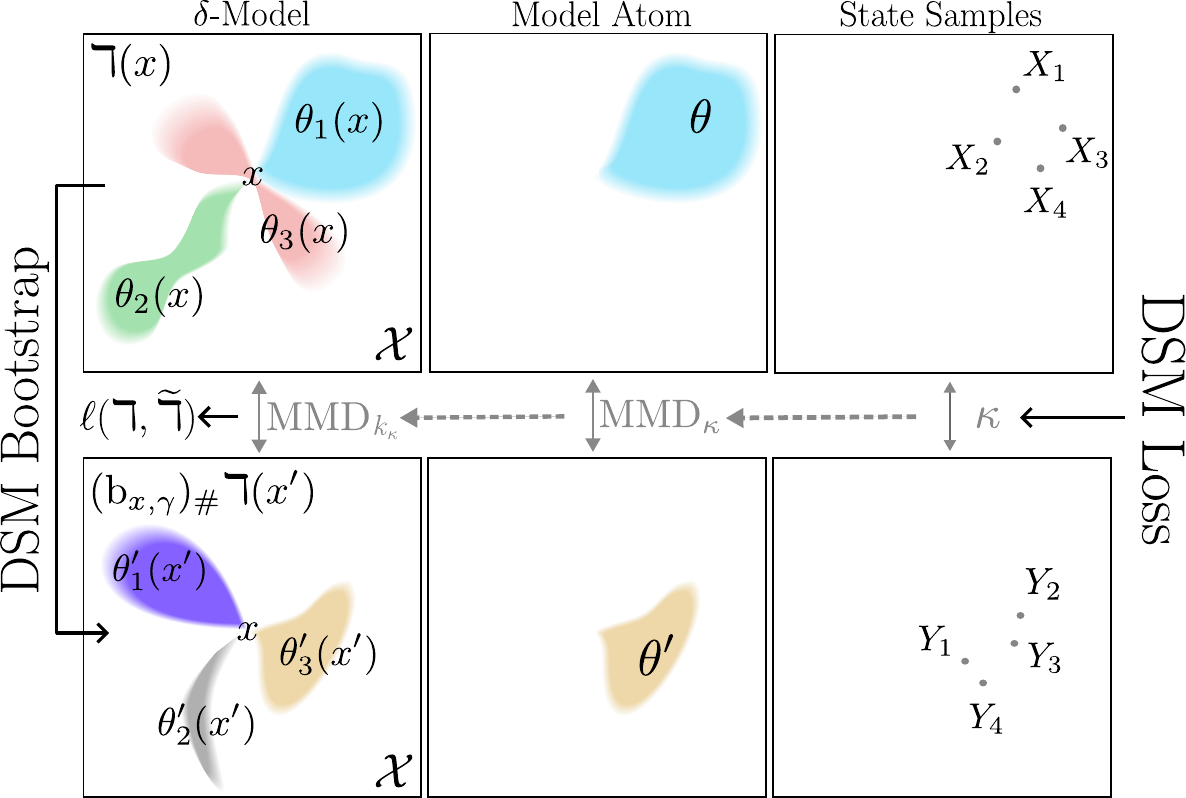}
    \caption{The components of a $\delta$-model (Section~\ref{sec:rep-learn-dsr:disrete}), and the kernels and distances involved in training them (Section~\ref{sec:rep-learn-dsr:mmd}).
    }
    \label{fig:delta-model}
\end{figure}
Since each atom $\theta_i(x)$ is a distribution over a potentially large space $\mathcal{X}$, we propose to represent the atoms as \emph{generative models}, in the spirit of $\gamma$-models \citep{janner2020gamma}. In practice, the generative models can be implemented with function approximators that take as input noise variables similar to the generator of a generative adversarial network \citep[GAN;][]{goodfellow14gan}.
We refer to such an EWP model as a \emph{\algobj}; Figure~\ref{fig:delta-model} illustrates its components.

\textbf{Terminology.} We have introduced two levels of probability distributions: $\daleth(x)$ is a distribution over the generative models $\{\theta_i(x)\}_{i=1}^m$; and each $\theta_i(x)$ is a distribution over the state space. To keep track of these two levels, we refer to $\dsr[](x)$ as a \textbf{model distribution} (that is, a distribution over generative models), and the generative models $\theta_i(x)$ as \textbf{state distributions} or \textbf{model atoms}. A generative model $\theta \sim \dsr[](x)$ distributed according to $\dsr[](x)$ is a \textbf{model sample}, while a state $X' \sim \theta$ sampled from a generative model is referred to as a \textbf{state sample}.

\subsection{Learning from Samples}\label{sec:rep-learn-dsr:mmd}

Our goal is to construct an algorithm for learning approximations of the distributions $\dsr(x)$, parameterized as \algobjs, from data. We construct a temporal-difference learning scheme \citep{sutton1984temporal,dayan1993improving} to approximately solve the distributional Bellman \Eqref{eq:randocc-bellman-equation} in this metric space, by updating our \algobj $\dsr[](x)$ to be closer to the transformation described by the right-hand side of the distributional Bellman equation in Proposition~\ref{prop:randocc:bellman}, that is
\begin{align}\label{eq:bellman-target}
    \widetilde{\dsr[]}(x) \defeq \expectation{X'\sim\statetransitionkernel(\cdot\mid x)}{\frac{1}{m} \sum_{i=1}^m \delta_{(1-\gamma)x + \gamma \theta_i(X')}} \, .
\end{align}
To define an update that achieves this, we will specify a loss function over the space occupied by $\dsr[](x)$ (namely $\mathscr{P}(\mathscr{P}(\mathcal{X}))$), distributions \emph{over} distributions of state); this requires care, since this space has such complex structure relative to standard distributional RL problems.
We propose to use the \emph{maximum mean discrepancy} \citep[MMD;][]{gretton2012kernel} to construct such a loss. 
We begin by recalling that for probability distributions $p,q$ over a set $\mathcal{Y}$, the MMD corresponding to the kernel $\kappa:\mathcal{Y}\times\mathcal{Y}\to\R$ is defined as
\begin{equation}\label{eq:mmd}
\small
    \begin{aligned}
    \mmd[\kappa]^2(p, q) =
    \expectation{}{\kappa(X, X') + \kappa(Y, Y') - 2\kappa(X, Y)}\\
    \quad(X, X')\sim p\otimes p,\ (Y, Y')\sim q\otimes q \, .
    \end{aligned}
\end{equation}

\textbf{State kernel.} To compare state distributions $\theta, \theta' \in \mathscr{P}(\mathcal{X})$, we will take a \textbf{state kernel} $\innerkernel : \mathcal{X} \times \mathcal{X} \rightarrow \mathbb{R}$, and aim to compute $\text{MMD}_{\innerkernel}(\theta, \theta')$. Since in $\delta$-models we represent state distributions $\theta, \theta'$ as generative models, we approximate the exact MMD in \Eqref{eq:mmd} by instead using samples from the generative models \citep[Eq. 3]{gretton2012kernel}. If we take $X_1,\ldots,X_{n_1} \overset{\mathrm{i.i.d.}}{\sim} \theta$, and $Y_1,\ldots,Y_{n_2} \overset{\mathrm{i.i.d.}}{\sim} \theta'$ independently, we obtain the following estimator for $\mmd[\innerkernel]^2(\theta_i, \theta_j)$:
{\small
\begin{align}\label{eq:inner-kernel-estimator}
        &\mmdsamples[\innerkernel]^2(X_{1:n_1}, Y_{1:n_2}) \defeq\\
        &\sum_{\substack{i,j=1 \\ i<j}}^{n_1} \frac{\innerkernel(X_i, X_j)}{\binom{n_1}{2}} + \sum_{\substack{i,j=1 \\ i< j}}^{n_2} \frac{\innerkernel(Y_i, Y_j)}{\binom{n_2}{2}}
        -2\sum_{i=1}^{n_1}\sum_{j=1}^{n_2} \frac{\innerkernel(X_i, Y_j)}{n_1n_2} \, .\tag*{}
\end{align}
}

\textbf{Model kernel.} \Eqref{eq:inner-kernel-estimator} uses the state kernel to define a metric between generative models. However, ultimately we need a loss function defined at the level of \emph{model distributions} $\dsr[](x)$, so that we can define gradient updates that move these quantities towards their corresponding Bellman targets (\Eqref{eq:bellman-target}). We now use our notion of distance between state distributions to define a kernel on $\mathscr{P}(\mathcal{X})$ itself, which will allow us to define an MMD over $\mathscr{P}(\mathscr{P}(\mathcal{X}))$, the space of model distributions. To do so, we follow the approach of \citet[Eq.~6]{ChrSte10} and \citet{SzaGrePocSri15} by defining a \textbf{model kernel} $\outerkernel{} : \mathscr{P}(\mathcal{X}) \times \mathscr{P}(\mathcal{X}) \rightarrow \mathbb{R}$ as a function of $\mmd[\innerkernel]$. In particular, for each $\theta, \theta' \in \mathscr{P}(\mathcal{X})$, we set
\begin{equation}\label{eq:outer-kernel}
    \outerkernel{}(\theta, \theta'; \sigma) = \rho\left({\mmd[\innerkernel](\theta, \theta')} / {\sigma}\right)
\end{equation}
for $\sigma>0$, where $\rho:y\mapsto (1 + y^2)^{-1/2}$ is the inverse multiquadric radial basis function. 
\citet[Table 1]{SzaGrePocSri15} shows that $\outerkernel{}$ is characteristic for this choice of $\rho$.

\textbf{DSM MMD loss.} We now specify a loss that will allow us to update $\dsr[]$ towards the Bellman target in \Eqref{eq:bellman-target}, by employing the MMD under the model kernel $\outerkernel{}$:
\begin{align}\label{eq:model-mmd}
    \textstyle
    \ell(\dsr[], \widetilde{\dsr[]}; x) = \text{MMD}^2_{\outerkernel{}}(\dsr[](x),\widetilde{\dsr[]}(x) ) \, .
\end{align}
To build a sample-based estimator of this loss, we take a sampled state transition $(x, x')$ generated by the policy $\pi$, and expand the MMD above in terms of evaluations of the kernel $\outerkernel{}$; writing $\theta_i(x) = (1-\gamma) \delta_x + \theta_i(x')$, this leads to the following loss for the $\delta$-model representation,
\begin{equation*}
    \small
    \begin{aligned}
    \frac{1}{m^2}
    \sum_{i,j=1}^m 
        \left[
        \outerkernel{}(
            \theta_i(x), \theta_j(x)
            ) %
        -2 \outerkernel{}(
            \theta_i(x),
            \bar{\theta}_j(\bar{x})
        )
        \right]
    \, . %
    \end{aligned}
\end{equation*}
Finally, to obtain a loss on which we can compute gradients in practice, each model kernel evaluation above be can be approximated via \Eqref{eq:outer-kernel}, with the resulting state kernel MMD estimated via \Eqref{eq:inner-kernel-estimator}. Note that we can sample from distributions of the form $(1-\gamma)\delta_x + \gamma \theta_i(x')$ by first sampling $Y \sim \text{Bernoulli}(1-\gamma)$, returning $x$ if $Y=1$, and otherwise returning an independent sample from $\theta_i(x')$.
See Figure~\ref{fig:delta-model} for an illustration of how the loss is constructed.

\section{Practical Training of \AlgObjs}\label{sec:deep-dsr}
Heretofore, we introduced a tractable representation and learning rule
for estimating the \absobj from data. This section highlights two techniques that are
crucial for stable learning; pseudocode is provided in Appendix \ref{appendix:alg}.

\subsection{$n$-step Bootstrapping}\label{sec:deep-dsr:bootstrapping}
The procedure outlined in Section~\ref{sec:rep-learn-dsr:mmd} computes \algobj targets
via one-step bootstrapping. In accordance with \Eqref{eq:randocc-bellman-equation}, the probability
mass of the targets due to bootstrapping is $\gamma$, which can be large when we are
concerned with long horizons.
Consequently, the signal-to-noise ratio in the targets is low,
which dramatically impedes learning.

Inspired by efforts to reduce the bias of bootstrapping in RL \citep{watkins1989learning,sutton2018reinforcement}, we compute $n$-step targets of the
\absobj. By \Eqref{eq:randocc-bellman-equation}, we have
\begin{equation}\label{eq:randocc:bellman:nstep}
\begin{aligned}
\randocc(\cdot\mid x) \eqlaw (1-\gamma)\sum_{i=0}^{n-1}\gamma^i\delta_{X_i} + \gamma^n\randocc(\cdot\mid X_n)
\end{aligned}
\end{equation}
where $X_{k+1}\sim \statetransitionkernel(\cdot\mid X_k)$ and $X_0=x$. 
An $n$-step version of the DSM MMD loss can then be obtained by replacing
the sampled one-step Bellman targets ${(1-\gamma) \delta_x + \gamma \dsr[](x')}$ in \Eqref{eq:model-mmd}
with the $n$-step target ${\sum_{k=0}^{n-1}(1-\gamma)\gamma^k \delta_{x_k} + \gamma^n \theta_i(x_n)}$.
In analogy with the one-step case, we can sample from this distribution by first sampling $Y$
from a Geometric($1-\gamma$) distribution, returning $x_k$ if $Y=k<n$, and returning a sample
from $\theta_i(x_n)$ otherwise.
By increasing $n$, we decrease the influence of 
bootstrap samples on the targets, leading to a stronger learning signal grounded in samples from the trajectory.

We found that training stability tends to improve substantially when bootstrap samples
account for roughly 80\% of the samples in the procedure above.
Appendix~\ref{appendix:experiments:ablation} includes a more detailed ablation on the choice of $n$. 
Notably, this procedure for computing TD targets for generative modeling of occupancy measures
is not specific to the \absobj or \algobjs.
We anticipate that this technique would generally be useful
for training geometric horizon models with longer horizons, which was reported to
be a major challenge \citep{janner2020gamma, thakoor2022generalised}.
\begin{figure*}
\centering
\begin{subfigure}{0.52\linewidth}
\centering
\includegraphics[width=\linewidth]{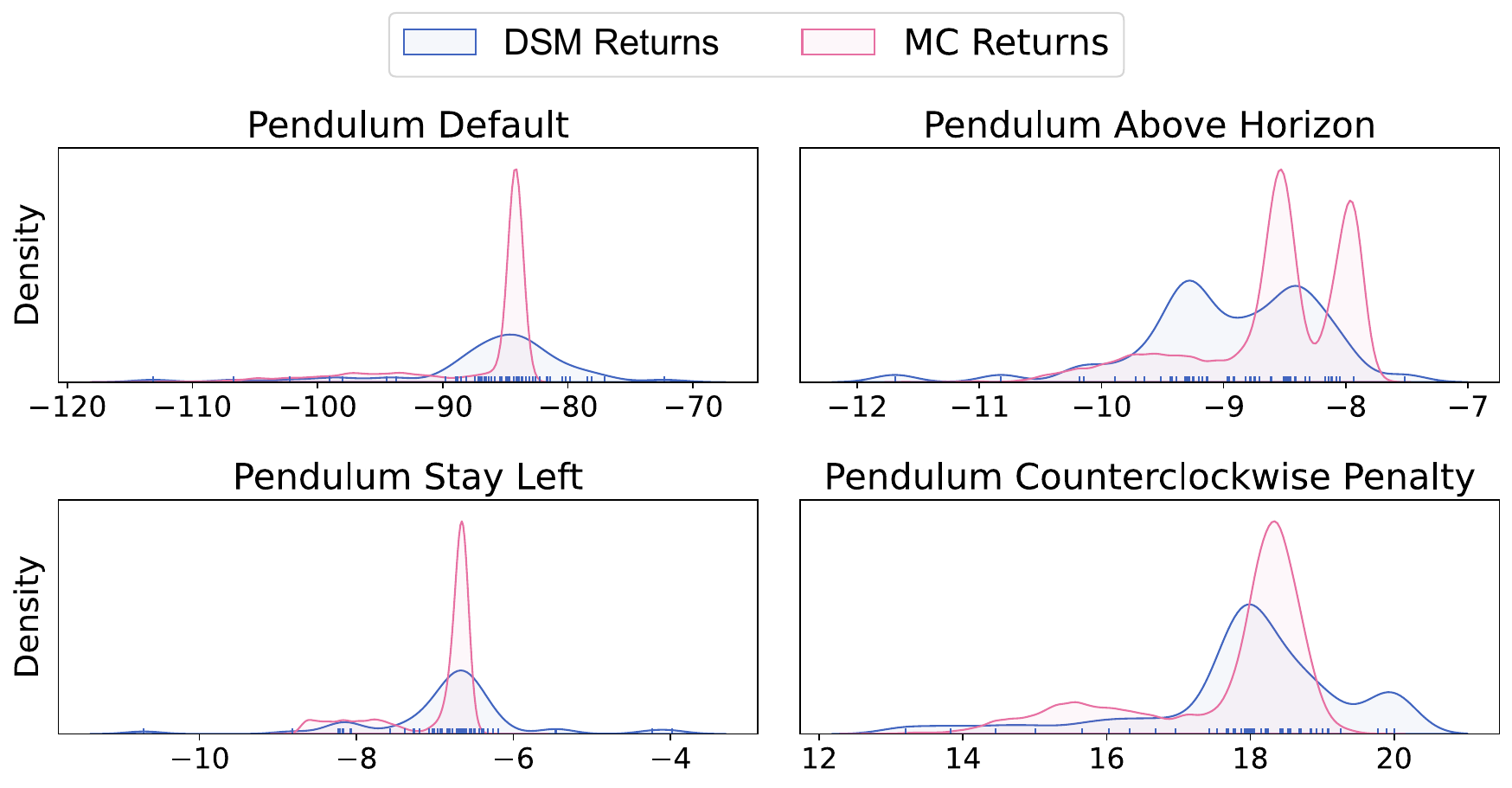}
\vspace{-0.5cm}
\subfiglabel{Return distribution predictions by \AbsObjAcronym.}\label{fig:exp:returns:pendulum:kde}
\end{subfigure}
\hfill
\begin{subfigure}{0.45\linewidth}
\centering
\vspace{0.45cm}
\includegraphics[width=\linewidth]{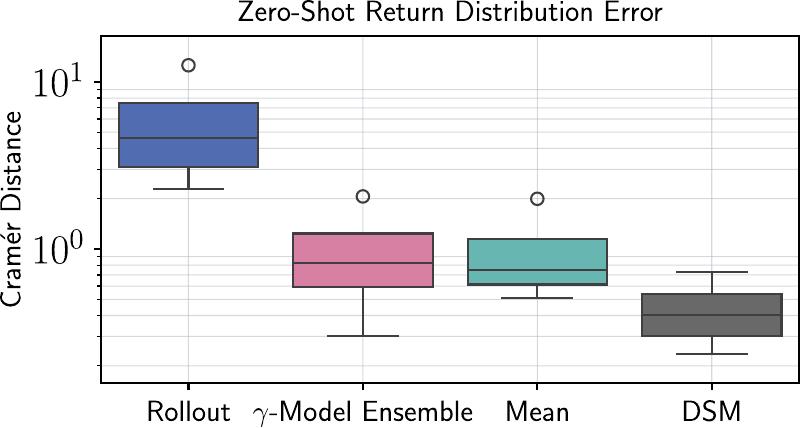}
\vspace{-0.5cm}
\subfiglabel{Quality of return distribution estimates.}\label{fig:exp:returns:pendulum:cramer}
\end{subfigure}
\label{fig:exp:returns:pendulum}
\end{figure*}

\subsection{Kernel Selection}\label{sec:deep-dsr:kernel}

When training a \algobj with bootstrapped targets, naturally the model/state
distributions comprising $\dsr$ are continually evolving. This poses a challenge when
selecting the kernels we use in practice,
since this non-stationarity prevents us from identifying an appropriate
similarity measure \emph{a priori}.
As such, we found it necessary to employ \emph{adaptive} kernels that evolve with the distributions being modeled.

Powerful methods in the literature involve adversarially learning a kernel over a space of parameterized
functions. The MMD-GAN \citep{mmdgan,binkowski18demystifying} demonstrates how to parameterize
characteristic kernels with deep neural networks. \citet{mmdgan} shows that
for any characteristic kernel $\innerkernel:\mathcal{Y}\times\mathcal{Y}\to\R_+$, the 
function $\innerkernel\circ f: (x, y)\mapsto\innerkernel(f(x), f(y))$ is itself a characteristic 
kernel when $f:\mathcal{X}\to\mathcal{Y}$ is injective. In their work, $f$ is parameterized as the 
encoder of an autoencoder network, where the autoencoder training encourages $f$ to be injective.

In the case of the \absobj, parameterizing the model kernel as an injection on the space of probability measures
is a major challenge. Rather, we parameterize an adversarial state kernel following the model of
\citet{mmdgan}, using an invertible neural network based on iResNet \citep{behrmann19iresnet}.
Unlike an autoencoder, this \emph{enforces} injectivity, and to our knowledge, no other work has
employed invertible neural networks for modeling an adversarial kernel.
It should be noted that the state kernel is itself defined as a parameter of the model kernel
used in the comparison of \algobjs\ -- thus, by adaptively learning the state kernel,
our model kernel is itself adaptive.

We also found that further adaptation of the model kernel through the bandwidth $\sigma$ improved training.
Our approach is based on the
\emph{median heuristic} for bandwidth selection in kernel methods 
\citep{DBLP:journals/jmlr/TakeuchiLSS06,gretton2012kernel}. Prior to computing the model
MMD, we choose $\sigma^2$ to be the median of the pairwise $\mmd[\innerkernel]^2$ between
the model atoms of $\dsr[](x)$ and those of the bootstrap target $\tilde{\dsr[]}(x)$.
Appendix~\ref{appendix:experiments:ablation} ablates on our choice of adaptive kernels.

\section{Experimental Results}\label{sec:results}

We evaluate our implementation of the \absobj on two domains, namely a stochastic ``Windy Gridworld" environment, where a pointmass navigates 2D continuous grid subject to random wind force that pushes
it towards the corners, and the Pendulum environment \citep{atkeson97pendulum}.
As a baseline, we compare our method to an ensemble of $\gamma$-models
\citep{janner2020gamma}, which is almost equivalent to a $\delta$-model, with the difference being that the individual $\gamma$-models of the ensemble are trained independently rather than coupled through the model MMD loss. We implement the $\gamma$-models with MMD-GAN, similarly to the 
individual model atoms of a \algobj. We train 
an ensemble of $m$ $\gamma$-models, where $m$ is the number of model atoms in the comparable
\algobj implementation of the \absobj.
Conceptually, the $\gamma$-model ensemble is expected to capture the \emph{epistemic uncertainty} over the SM,
while the \absobj estimates the aleatoric uncertainty due to randomness of the MDP dynamics and the policy.
Alternatively, one can learn a model of the transition kernel $\statetransitionkernel$ and estimate return distributions by rolling out trajectories from the learned model and computing the discounted returns for those trajectories. Thus, for the purpose of quantitative evaluation, we additionally train a model of $\statetransitionkernel$ (training details are discussed in Appendix \ref{appendix:experiments:baselines}) and compare the accuracy of
zero-shot return distribution prediction by the \absobj to those estimated by rolling out trajectories from the forward model as described. Note, however, that beyond any difference in estimation
quality, the \absobj presents a major computational advantage over estimation via $\statetransitionkernel$: with the \absobject, it is not necessary to sample
long episode trajectories.

\textbf{Visualizing model atoms.}
In Figure \ref{fig:exp:windy:srdsr}, we examine the model atoms predicted by our implementation
of the \algobj trained on data from a uniform random policy in the Windy Gridworld. Due to the
nature of the wind in this domain, which always forces the agent to the corner of the quadrant
where it is located, a uniform random policy exhibits a multimodal distribution of model atoms,
as shown by the colored densities in the top-left. Alternatively, when examining an ensemble of
$\gamma$-models trained on the same data, we see that the models in the ensemble all predict similar
state occupancies which align closely with the SM -- crucially, only the \absobj captures the
diversity of ``futures" that the agent can experience.

\textbf{Zero-shot policy evaluation.}
A unique feature of the \absobj is that it acts as an operator that transforms reward functions
into return distribution functions. We explore the distributions over returns predicted by the
\absobj for several held-out reward functions and analyze their similarity with return
distributions estimated by Monte Carlo. Figure \ref{fig:exp:returns:pendulum:kde} showcases return
distributions predicted by the \absobj on four tasks in the Pendulum environment meant to model constraints that may be imposed on the system (\texttt{Default}, \texttt{Above Horizon}, \texttt{Stay Left}, \texttt{Counterclockwise Penalty}; details in Appendix \ref{sec:experiments:pendulum}). We can see that these predictions capture important statistics, such as the mode and the support of the
distributions, which could not be captured by point estimates of the return. Similar results in Windy Gridworld are shown in Appendix \ref{sec:experiments:windy}. For quantitative
evaluation, Figure \ref{fig:exp:returns:pendulum:cramer} reports the quality of the return distribution predictions by their dissimilarity
to the return distributions estimated by Monte Carlo according to the Cram\'er distance \citep{szekely2013energy,bellemare2017cramer}.
We compare the \AbsObjAcronym predictions to those computed by three baselines: return 
distributions estimated by sampling rollouts from a learned transition kernel (labeled 
$\mathsf{Rollout}$), return distributions imputed from value function predictions among
an ensemble of $\gamma$-models, and return distributions constructed by placing a Dirac
mass at the MC expected return (labeled $\mathsf{Mean}$).
Among these baselines and existing methods in the literature, only $\mathsf{Rollout}$ can 
produce proper return distribution estimates \emph{in principal}. However, we find that 
accumulation of error throughout sampled trajectories prevents this model from 
achieving reasonable return distributions, which is consistent with the difficulties of 
accurately rolling out long trajectories from learned forward models \citep{jafferjee2020hallucinating,abbas2020selective,lambert2022investigating}. 
While the ensemble of $\gamma$-models is not modeling the aleatoric uncertainty
of occupancy measures, we find that its superior ability to model long-horizon behavior
enables it to estimate return distributions more accurately, achieving similar
quality to a Dirac mass centered at the ground truth mean return. The DSM predictions
substantially outperform all baselines, demonstrating that the proposed \algobj retains
the long-horizon consistency of $\gamma$-models, while additionally providing aleatoric
uncertainty estimates far beyond the capabilities of the learned~$\statetransitionkernel$.

\begin{figure*}
\centering
\begin{tabular}{r@{\hskip12pt} | l}
\begin{subfigure}{0.315\linewidth}
\includegraphics[width=\linewidth,trim=5cm -2.5cm 0 0,clip]{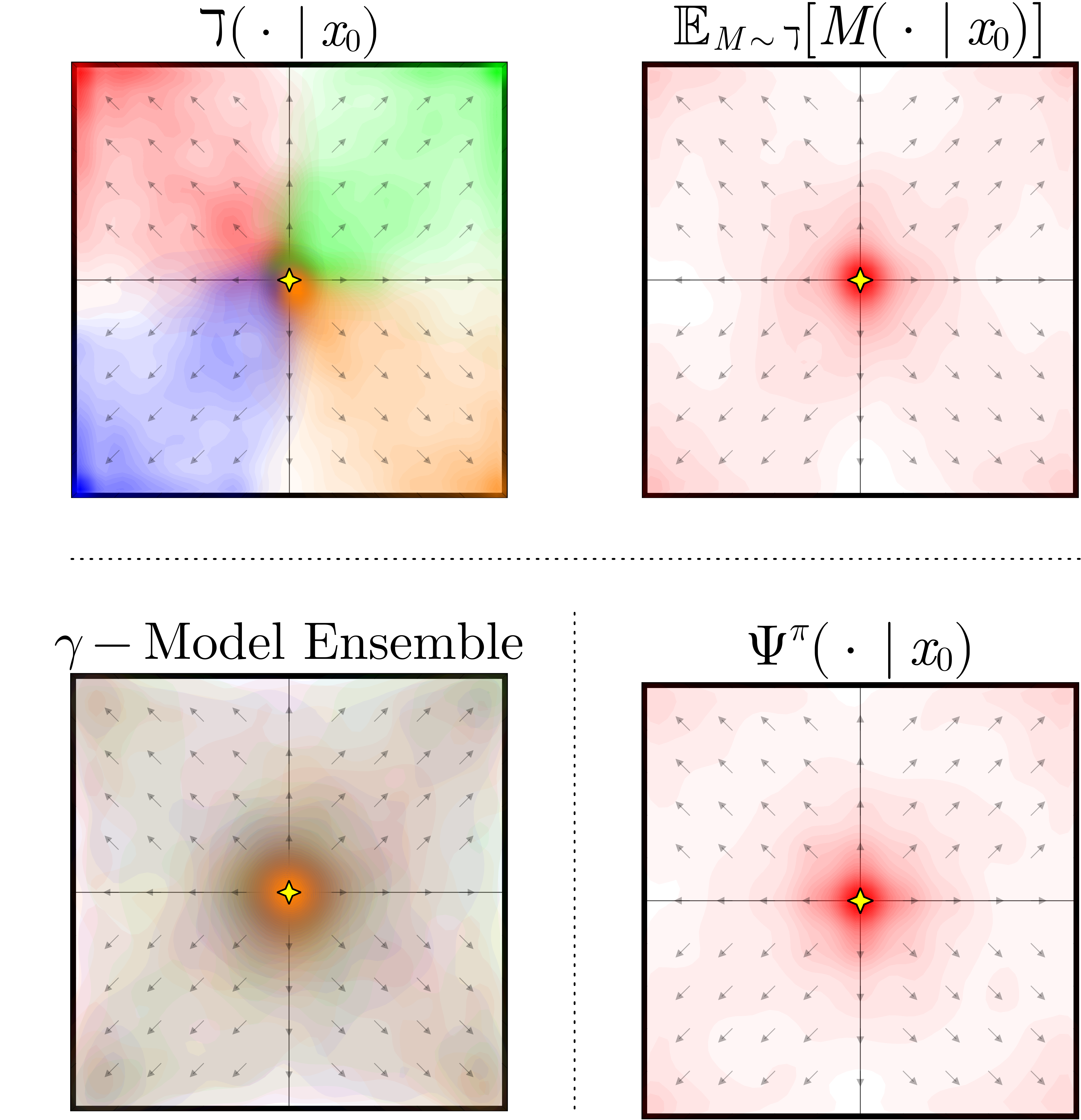}
\caption{}\label{fig:exp:windy:srdsr}
\end{subfigure}
&
\begin{subfigure}{0.635\linewidth}
\includegraphics[keepaspectratio,width=\linewidth]{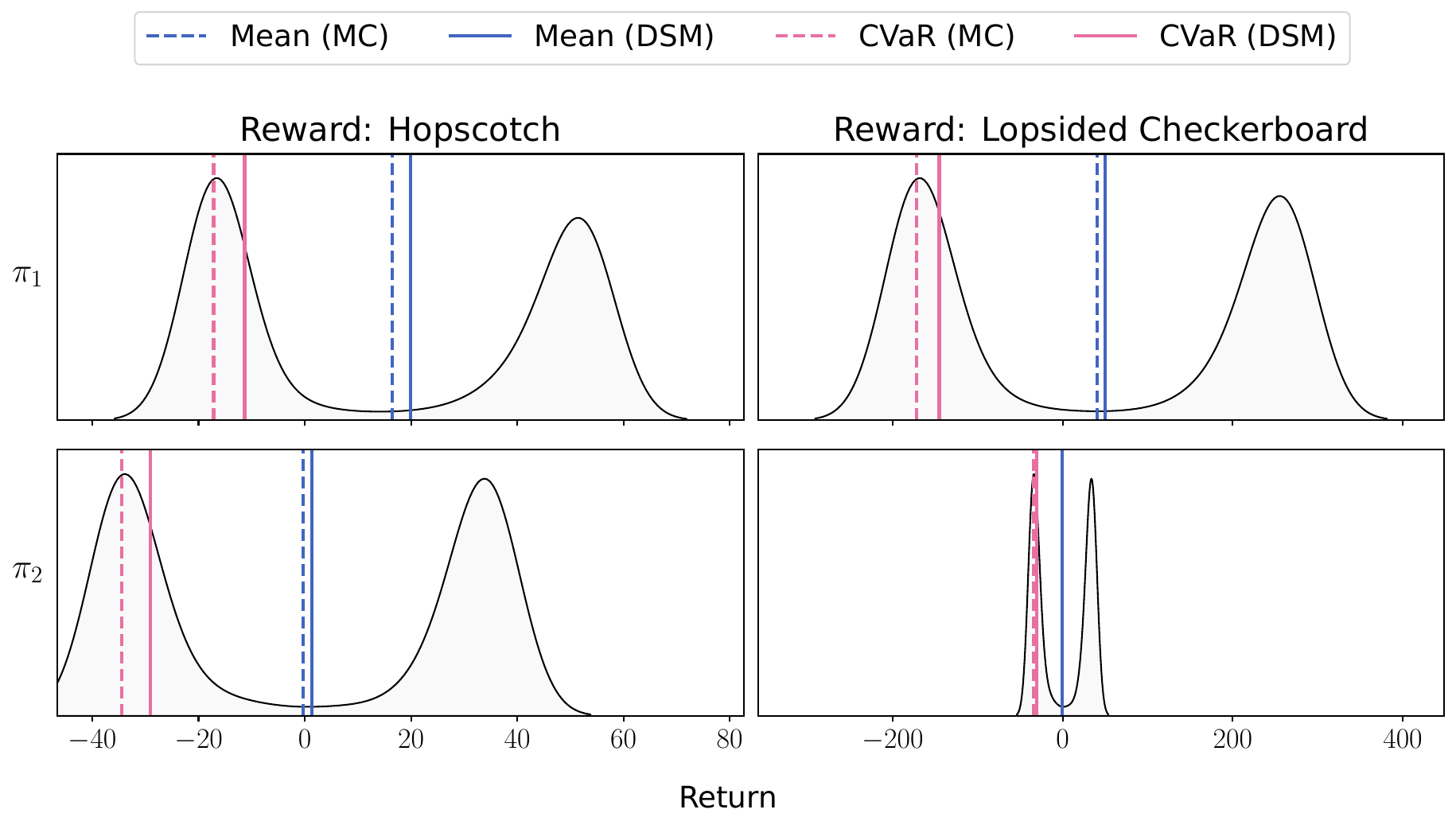}
\caption{}\label{fig:exp:windy:policysel}
\end{subfigure}
\end{tabular}
\caption{\AbsObject predictions in Windy Gridworld. \textbf{(\ref{fig:exp:windy:srdsr})}: Figures in the left column show the 
model atoms predicted by the \absobj (distinguished by color) and by an ensemble of $\gamma$-models. Figures in the right column show the mean over \absobj model atoms and the SM itself. \textbf{(\ref{fig:exp:windy:policysel})}: \AbsObj predictions of return statistics on held-out reward functions for two policies, $\pi_1,\pi_2$. For each reward function, the \absobj correctly ranks policies with respect to both mean and CVaR.}\label{fig:exp:windy}
\end{figure*}

\textbf{Risk-sensitive policy selection.}
Finally, we demonstrate that \absobjs can be used to effectively rank policies by various risk-sensitive
criteria on held-out reward functions. In Figure \ref{fig:exp:windy:policysel}, we train \absobjs
for two different policies, and use them to predict return distributions for two reward functions.
We focus on two functionals of these return distributions, namely the mean and the conditional value
at risk at level $0.4$ \citep[][$0.4$-CVaR]{rockafellar2002conditional}. We see that for both reward
functions, the \absobj accurately estimates both functionals, and is able to correctly identify
the superior policy for each criterion. Particularly, for the \texttt{Lopsided Checkerboard} reward,
the \absobj identifies $\pi_1$ as superior with respect to mean reward (identified by locations of
solid blue lines), and alternatively identifies $\pi_2$ as superior with respect to $0.4$-CVaR
of the return (identified by locations of the solid pink lines). These rankings are validated by
the locations of the dashed lines, which are computed by Monte Carlo. We note that, to our
knowledge, \emph{no other method can accomplish this feat}.
On the one hand, existing distributional RL
algorithms could not evaluate the return distributions for held-out reward functions.
On the other hand, any algorithm rooted in the SM for zero-shot
evaluation can only rank policies by their mean returns, so they must fail to rank $\pi_1,\pi_2$ by at least one of the objectives.

\section{Related Work}\label{sec:related-work}
The successor representation \citep[SR;][]{dayan1993improving}, originally proposed for finite-state MDPs, has recently been expanded to continuous spaces by leveraging both generative models \citep{janner2020gamma,thakoor2022generalised} and density models \citep{blier2021learning,blier2022some}.
Successor features \citep[SFs;][]{barreto2016successor, barreto2020fast} generalize the SR by modelling a discounted sum of state \emph{features}. %
These models are notable for their ability to perform both zero-shot policy evaluation \citep{dayan1993improving,barreto2016successor} and optimization \citep{borsa2018universal,touati2021learning,touati2023does}.

Closely related to our work is that of \citet{gimelfarb2021risk} and \citet{carvalho2023combining}, which applied the distributional RL techniques of \citet{Bellemare2017ADP} and \citet{achab2023one} to learn categorical distributions over features, but did not account for their joint distribution. \citet{gimelfarb2021risk} uses these distributions to optimize an entropic risk objective.
These approaches also bears a close relationship with the emerging field of multivariate distributional RL \citep{freirich2019distributional,zhang2021distributional}, which does learn joint distributions over finite dimensional features. In particular \citet{zhang2021distributional} make use of an MMD loss for learning multivariate return distributions, building on the scalar approach of \citet{nguyen2020distributional}.
Finally, \citet{vertes2019neurally} consider the task of learning the SR in POMDPs, which they~referred~to~as~a~distributional~SR.

Beyond transferring knowledge across tasks, learning long-term temporal structure can enhance the 
representation quality of function approximators for individual sequential decision-making problems 
\citep{lelan23subspace,farebrother2023proto, ghosh2023reinforcement}, guiding exploration 
\citep{jinnai2019cover,machado20sr,jain2024maximum}, modeling temporal abstraction 
\citep{machado2018eigen,machado2023temporal}, improving off-policy estimation \cite{nachum2020reinforcement,fujimoto2021deep}, imitation learning \cite{sikchi2023dual,pirotta24bfmil}, and amortizing planning 
\citep{eysenbach2020c,eysenbach2022contrastive,thakoor2022generalised}, 
as well as other forms of risk-sensitive decision making \citep{zhang2021cautious}.
The successor representation also plays a key explanatory role in understanding generalization in RL  \citep{lelan2022generalization,lelan23bootstrap}.
Additionally, both distributional RL 
\citep{dabney2020distributional,lowet2020distributional}  
and successor representations
\citep{stachenfeld14sr,stachenfeld2017hippocampus,momennejad2017successor}
have been shown to provide plausible models for~biological~phenomena~in~the~brain.

\section{Conclusion}\label{sec:discussion}

This paper presents a fundamentally new approach to distributional RL, which factorizes return distributions into components comprising the immediate reward function 
and the \emph{\absobject}. This factorisation reveals the prospect of 
zero-shot distributional policy evaluation. Notably, this enables efficient comparisons between policies on unseen tasks with respect to arbitrary risk criteria, which no other existing methods have demonstrated. We have also presented a tractable algorithmic framework for 
training $\delta$-models, which approximate the distributional SM with diverse 
generative models, and have identified crucial techniques for large-scale
training  of $\delta$-models in practice.

\section*{Acknowledgements}
The authors would like to thank Eric Zimmermann, Diana Borsa, Marek Petrik, Erick Delage, Nathan U. Rahn, Pierluca~D'Oro, Arnav Jain, Max Schwarzer, Igor Mordatch, and Pablo Samuel Castro for their invaluable discussions and feedback on this work.
This work was supported by the Fonds de Recherche du Qu\'ebec, the National Sciences and Engineering Research Council of Canada (NSERC), Calcul Qu\'ebec, the Digital Research Alliance of Canada, and the Canada CIFAR AI Chair program.

We would also like to thank the Python community whose contributions made this work possible. In particular, this work made extensive use of Jax \citep{jax}, Flax \citep{flax}, Optax \citep{deepmind2020jax}, EinOps \citep{rogozhnikov2022einops}, and Seaborn \citep{waskom2021seaborn}.

\section*{Impact Statement}
This paper presents work whose goal is to advance the field of Machine Learning. There are many potential societal consequences of our work, none which we feel must be specifically highlighted here.
{
\hypersetup{urlcolor=black}
\bibliography{references}
\bibliographystyle{icml2024/icml2024}
}

\clearpage

\appendix
\onecolumn

\section{Algorithm}\label{appendix:alg}

In this section, we restate the core $\delta$-model update derived in Section~\ref{sec:rep-learn-dsr}, including the $n$-step bootstrapping and adversarial kernel modifications described in Section~\ref{sec:deep-dsr}.
\ifdefined\isaccepted
Source code is provided at \url{https://github.com/jessefarebro/distributional-sr}.
\fi

\begin{center}
\scalebox{1.0}{
\begin{minipage}{\linewidth}
\begin{algorithm}[H]
    \begin{algorithmic}
        \STATE{\textbf{Require: } Policy $\pi$ with stationary distribution $d_\pi$, GAN generator $\Phi$, GAN parameters $\{\zeta_i\}_{i=1}^m$ and target parameters $\{\overline{\zeta}_i\}_{i=1}^m$, discriminator function $f$, adversarial kernel parameters $\{\xi_i\}_{i=1}^m$, kernel $\kappa$, step sizes $\alpha,\lambda$, number of state samples $s$.}
        \WHILE{training}
            \STATE Set $\xi_i\leftarrow \xi_i + \alpha\nabla_{\xi_i}\ell(\{\zeta_j\}_{j=1}^m,\{\bar{\zeta}_j\}_{j=1}^m, \{\xi_j\}_{j=1}^m)$ for $i=1,\dots, m$\COMMENT{Discriminator: maximize model MMD}
            \STATE Set $\zeta_i\leftarrow \zeta_i - \alpha\nabla_{\zeta_i}\ell(\{\zeta_j\}_{j=1}^m,\{\bar{\zeta}_j\}_{j=1}^m, \{\xi_j\}_{j=1}^m)$ for $i = 1,\dots,m$ \COMMENT{Generator: minimize model MMD}
            \STATE Set $\bar{\zeta}_i\leftarrow (1-\lambda)\bar{\zeta}_i + \lambda\zeta_i$ for $i=1,\dots,m$ \COMMENT{Generator: target parameter update}
        \ENDWHILE
        \STATE
        \FUNCTION{$\ell(\{\zeta_i\}_{i=1}^m, \{\bar{\zeta}_i\}_{i=1}^m, \{\xi_i\}_{i=1}^m)$}
        \STATE Sample $x_1\sim d_\pi, x_k \sim P^\pi(\cdot\mid x_{k-1})$ for $k=2,\ldots,n$.
        \FOR{$i=1,\ldots,m$}
            \STATE Sample $z_i^1,\ldots, z_i^{s}$ i.i.d.\ from GAN noise distribution
            \STATE Set $x_i^j \leftarrow \Phi(z_i^j; x_1, \zeta_i)$ for $j=1,\ldots,s$
            \STATE Sample $\omega_i^1,\ldots,\omega_i^{s}$ i.i.d.\ from GAN noise distribution
            \STATE \quad and $Y_i^1,\ldots,Y_i^{s}$ i.i.d.\ from $\text{Geometric}(1-\gamma)$
            \STATE Set $\bar{x}_i^j \leftarrow \Phi(\omega_i^j; x_n, \bar{\zeta}_i)$ if $Y_i^j\geq n$, else set $\bar{x}_i^j \leftarrow x_{Y_i^j}$, for  $j=1,\ldots,s$
            \STATE Set $y_i^j\leftarrow f(x_i^j, \xi_i)$ for $j=1,\ldots s$
            \COMMENT{Adversarial Kernel Transformations}
            \STATE Set $\bar{y}_i^j\leftarrow f(\bar{x}_i^j, \xi_i)$ for $j=1,\ldots s$
        \ENDFOR
        \FOR[MMDs Between Source Model Atoms]{$i=1,\dots,m$}%
        \FOR{$i'=1,\ldots,m$}
        \STATE Set $\displaystyle d^s_{i,i'}\leftarrow \frac{1}{\binom{s}{2}}\sum_{\substack{l,k=1\\l < k}}^s\innerkernel(y_i^k, y_i^l) + \frac{1}{\binom{s}{2}}\sum_{\substack{l,k=1\\l < k}}^s\innerkernel(y_{i'}^k,y_{i'}^l) - \frac{2}{s^2}\sum_{l,k=1}^s\innerkernel(y_i^k, y_{i'}^l)$
        \COMMENT{\Eqref{eq:inner-kernel-estimator}}
        \ENDFOR
        \ENDFOR
        \FOR[MMDs Between Target Model Atoms]{$i=1,\dots,m$}%
        \FOR{$i'=1,\ldots,m$}
        \STATE Set $\displaystyle d^t_{i,i'}\leftarrow \frac{1}{\binom{s}{2}}\sum_{\substack{l,k=1\\l < k}}^s\innerkernel(\bar{y}_i^k, \bar{y}_i^l) + \frac{1}{\binom{s}{2}}\sum_{\substack{l,k=1\\l < k}}^s\innerkernel(\bar{y}_{i'}^k,\bar{y}_{i'}^l) - \frac{2}{s^2}\sum_{l,k=1}^s\innerkernel(\bar{y}_i^k, \bar{y}_{i'}^l)$
        \ENDFOR
        \ENDFOR
        \FOR[MMDs Across Source and Target Model Atoms]{$i=1,\dots,m$}%
        \FOR{$i'=1,\ldots,m$}
        \STATE Set $\displaystyle d^{st}_{i,i'}\leftarrow \frac{1}{s^2}\sum_{l,k=1}^s\innerkernel(y_i^k, \bar{y}_i^l) + \frac{1}{s^2}\sum_{l,k=1}^s\innerkernel(y_{i'}^k,\bar{y}_{i'}^l) - \frac{2}{s^2}\sum_{l,k=1}^s\innerkernel(y_i^k, \bar{y}_{i'}^l)$
        \ENDFOR
        \ENDFOR
        \STATE Set $\sigma^2 = \mathrm{Median}\left(\mathrm{Concat}\left(\{d^s_{i,i'}\}, \{d^t_{i,i'}\}, \{d^{st}_{i,i'}\}\right)\right)$\COMMENT{Adaptive Model Kernel Bandwidth}
        \STATE Set $\displaystyle L\leftarrow \frac{1}{m^2}\sum_{i,j=1}^m\left(\rho(\sqrt{d^s_{i,j}/\sigma^2}) + \rho(\sqrt{d^t_{i,j}/\sigma^2}) - 2\rho(\sqrt{d^{st}_{i,j}/\sigma^2})\right)$\COMMENT{Model MMD}
        \ENDFUNCTION
    \end{algorithmic}
    \caption{\algobj update.}\label{alg:dsr}
\end{algorithm}
\end{minipage}
}
\end{center}
\vfill

\section{Proofs}\label{sec:proofs}
\randocclinear*

\begin{proof}
This result can be verified by a direct calculation. Invoking Definition~\ref{def:randocc}, we have
\begin{align*}
    (M^\pi r)(x)
    &= \int_{\mathcal{X}}r(x')M^\pi(\mathrm{d}x'\mid x)\\
    &\eqlaw \Conditional{\int_{\mathcal{X}}\sum_{t\geq 0}(1-\gamma)\gamma^t r(x') \delta_{X_t}(\mathrm{d}x')}{X_0=x} \\
    &\eqlaw (1-\gamma)\Conditional{\sum_{t\geq 0}\gamma^tr(X_t)}{X_0=x}\\
    &\eqlaw(1-\gamma)\randret[r]{\pi}(x)
\end{align*}
where the third step invokes Fubini's theorem subject to the boundedness of $r$.
The claimed result simply follows by dividing through by $1-\gamma$.
\end{proof}

\randoccbellman*

\begin{proof}
    By Definition \ref{def:randocc}, we have
    \begin{align*}
        M^\pi(\cdot \mid x)
        &\eqlaw (1-\gamma)\delta_x + \sum_{t=1}^\infty (1 - \gamma) \gamma^t\delta_{X_t} \\
        &\eqlaw (1-\gamma)\delta_x + \gamma \sum_{t=0}^\infty(1-\gamma)\gamma^t\delta_{X_{t+1}} \\
        &\eqlaw (1-\gamma)\delta_{x} + \gamma M^\pi( \cdot \mid X') \, ,
    \end{align*}
    with the final equality in distribution following from the Markov property.
\end{proof}

\subsection{Distributional Dynamic Programming}\label{sec:additional_results:ddp}
In this section, we demonstrate how the distributional SM can be computed by dynamic programming. Following familiar techniques in the analysis of dynamic programming algorithms,
we will demonstrate that the distributional SM is the unique fixed point of a contractive operator, and appeal to the Banach fixed point theorem.

To begin, we will define the operator of interest, which we refer to as the distributional
Bellman operator $\mathcal{T}^\pi:\probset{\occspace}^{\mathcal{X}}\to\probset{\occspace}^{\mathcal{X}}$,

\begin{align*}
    (\mathcal{T}^\pi\daleth)(x) = \expectation{X'\sim\statetransitionkernel(\cdot\mid x)}{(\bootpush{x})_\sharp\daleth(X')}.
\end{align*}

It follows directly from \Eqref{eq:dsr:bellman} that $\dsr=\mathcal{T}^\pi\dsr$.

\dsmcontraction*
\begin{proof}
    Our approach is inspired by the coupling approach proposed by \citet{amortila2020distributional}. Denote by $\Pi(p, q)$ the set of couplings between distributions $p, q$.
    
    Let $\Gamma_{1,x'}\in\Pi(\daleth_1(x'), \daleth_2(x'))$ denote an $\epsilon$-optimal coupling
    with respect to the Wasserstein distance $W$, in the sense that
    
    \begin{align*}
        \int_{\probset{\mathcal{X}}}\int_{\probset{\mathcal{X}}}w_d(p, q)\Gamma_{1,x'}(\mathrm{d}p\times \mathrm{d}q)
        \leq W(\daleth_1(x'), \daleth_2(x')) + \epsilon
    \end{align*}
    
    for arbitrary $\epsilon>0$. Firstly, we note that $\Gamma_1\in\Pi((\mathcal{T}^\pi\daleth_1)(x), (\mathcal{T}^\pi\daleth_2)(x))$, where
    
    \begin{align*}
        \Gamma_1 &= \int_{\mathcal{X}}\statetransitionkernel(\mathrm{d}x'\mid x)
        \left[(\bootpush{x}, \bootpush{x})_\sharp\Gamma_{1,x'}\right].
    \end{align*}
    
    Here, $(\bootpush{x},\bootpush{x})_\sharp\Gamma_{1,x'}(A\times B) = \Gamma_{1,x'}(\bootpush{x}^{-1}(A)\times\bootpush{x}^{-1}(B))$ for measurable $A,B\subset\probset{\mathcal{X}}$. To see this, we note that for any measurable $P\subset\probset{\mathcal{X}}$,
    
    \begin{align*}
        \Gamma_1(P\times\probset{\mathcal{X}})
        &= \int_{\mathcal{X}}\statetransitionkernel(\mathrm{d}x'\mid x)\Gamma_{1,x'}(\bootpush{x}^{-1}(P)\times \probset{\mathcal{X}})\\
        &= \int_{\mathcal{X}}\statetransitionkernel(\mathrm{d}x'\mid x)\daleth_1(x')(\bootpush{x}^{-1}(P))\\
        &= \int_{\mathcal{X}}\statetransitionkernel(\mathrm{d}x'\mid x)\left[(\bootpush{x})_\sharp\daleth_1(x')\right](P)\\
        &\equiv \left[(\mathcal{T}^\pi\daleth_1)(x)\right](P),
    \end{align*}
    
    so that the first marginal of $\Gamma_1$ is $(\mathcal{T}^\pi\daleth_1)(x)$. Likewise, the second marginal of $\Gamma_1$ is $(\mathcal{T}^\pi\daleth_2)(x)$,
    confirming that $\Gamma_1$ is a coupling between $(\mathcal{T}^\pi\daleth_1)(x)$ and
    $(\mathcal{T}^\pi\daleth_2)(x)$. It follows that
    
    \begin{align*}
        \overline{W}(\mathcal{T}^\pi\daleth_1, \mathcal{T}^\pi\daleth_2)
        &= \sup_{x\in\mathcal{X}}W((\mathcal{T}^\pi\daleth_1)(x), (\mathcal{T}^\pi\daleth_2)(x))\\
        &\leq \sup_{x\in\mathcal{X}}\int_{\occspace}\int_{\occspace}w_d(p, q)\Gamma_1(\mathrm{d}p\times\mathrm{d}q)\\
        &= \sup_{x\in\mathcal{X}}\int_{\occspace}\int_{\occspace}\int_{\mathcal{X}}w_d(p, q)\statetransitionkernel(\mathrm{d}x'\mid x)\left[(\bootpush{x},\bootpush{x})_\sharp\Gamma_{1,x'}\right](\mathrm{d}p\times\mathrm{d}q)\\
        &= \sup_{x,x'\in\mathcal{X}}\int_{\occspace}\int_{\occspace}w_d(\bootpush{x}(p), \bootpush{x}(q))\Gamma_{1,x'}(\mathrm{d}p\times\mathrm{d}q).
    \end{align*}
    
    We now claim that $w_d(\bootpush{x}(p), \bootpush{x}(q))\leq\gamma w_d(p, q)$ for any
    $p,q\in\probset{\mathcal{X}}$. To do so, let $\Gamma_2\in\Pi(p, q)$ be an optimal coupling
    with respect to $w_d$, which is guaranteed to exist since $(\mathcal{X},d)$ is a Polish space
    \citep{villani2008optimal}. Define $\Gamma_3\in\probset{\mathcal{X}\times\mathcal{X}}$ such that
    
    \begin{align*}
        \Gamma_3 &= (1-\gamma)\delta_{(x, x)} + \gamma\Gamma_2.
    \end{align*}
    
    It follows that, for any measurable $X\subset\mathcal{X}$,
    \begin{align*}
        \Gamma_3(X\times\mathcal{X})
        &= (1-\gamma)\delta_{(x,x)}(X\times\mathcal{X})+ \gamma\Gamma_2(X\times\mathcal{X})\\
        &= (1-\gamma)\delta_x(X) + \gamma\Gamma_2(X\times\mathcal{X})\\
        &= (1-\gamma)\delta_x(X) + \gamma p(X)\\
        &= \bootpush{x}(p)(X)
    \end{align*}
    
    which confirms that $\bootpush{x}(p)$ is the first marginal of $\Gamma_3$. An analogous argument for the second marginal shows that $\Gamma_3$ is a coupling between
    $\bootpush{x}(p), \bootpush{x}(q)$. So, we see that
    
    \begin{align*}
        w_d(\bootpush{x}(p), \bootpush{x}(q))
        &= \inf_{\Gamma\in\Pi(\bootpush{x}(p), \bootpush{x}(q))}\int_{\mathcal{X}}\int_{\mathcal{X}}d(y, y')\Gamma(\mathrm{d}y\times\mathrm{d}y')\\
        &\leq \int_{\mathcal{X}}\int_{\mathcal{X}}d(y, y')\Gamma_3(\mathrm{d}y\times\mathrm{d}y')\\
        &= (1-\gamma)d(x, x)+\gamma \int_{\mathcal{X}}\int_{\mathcal{X}}d(y, y')\Gamma_2(\mathrm{d}y\times\mathrm{d}y')\\
        &= \gamma w_d(p, q).
    \end{align*}
    
    Now, continuing the bound from earlier, we have
    
    \begin{align*}
        \overline{W}(\mathcal{T}^\pi\daleth_1, \mathcal{T}^\pi\daleth_2)
        &\leq \sup_{x,x'\in\mathcal{X}}\int_{\occspace}\int_{\occspace}w_d(\bootpush{x}(p), \bootpush{x}(q))\Gamma_{1,x'}(\mathrm{d}p\times\mathrm{d}q)\\
        &\leq \gamma\sup_{x\in\mathcal{X}}\int_{\occspace}\int_{\occspace}w_d(p, q)\Gamma_{1,x}(\mathrm{d}p\times\mathrm{d}q)\\
        &\leq\gamma\sup_{x\in\mathcal{X}}\left[W(\daleth_1(x), \daleth_2(x)) + \epsilon\right]\\
        &= \gamma\overline{W}(\daleth_1, \daleth_2) + \gamma\epsilon
    \end{align*}
    
    Thus, since $\epsilon>0$ was arbitrary, the claim follows.
\end{proof}

\dsmbanach*
\begin{proof}
    Prior to applying the Banach fixed point theorem it is necessary to ensure that $\overline{W}$ is finite on $\probset{\occspace}^{\mathcal{X}}$ to ensure that a fixed point will be reached.
    Since $\mathcal{X}$ is compact and metrics are continuous, it follows that the metric $d$ is bounded over
    $\mathcal{X}$, that is,
    
    \begin{align*}
        \sup_{x,y\in\mathcal{X}}d(x, y)\leq C < \infty
    \end{align*}
    
    for some constant $C$. As such, the Wasserstein distance $w_d$, as an expectation over distances measured by $d$, is also bounded by $C$, and following the same logic, the metrics
    $W,\overline{W}$ are bounded by $C$. Then, since it is clear from \eqref{eq:dsr:bellman} and \eqref{eq:distributional-bellman-operator:dsm} that $\dsr=\mathcal{T}^\pi\dsr$, we have
    
    \begin{align*}
        \overline{W}(\daleth_k, \dsr)
        &= \overline{W}(\mathcal{T}^\pi\daleth_{k-1}, \dsr) \\
        &= \overline{W}(\mathcal{T}^\pi\daleth_{k-1}, \mathcal{T}^\pi\dsr) \\
        &\leq \gamma\overline{W}(\daleth_{k-1}, \dsr) \\
    \end{align*}
    
    where the final step leverages the contraction provided by Proposition \ref{prop:dsm:contraction}. Then, repeating $k-1$ times, we have
    
    \begin{align*}
        \overline{W}(\daleth_k, \dsr)
        &\leq \gamma^k\overline{W}(\daleth_0, \dsr) \\
        &\leq \gamma^kC
    \end{align*}
    
    Since $|\gamma|<1$ and $C$ is finite, it follows that $\overline{W}(\daleth_k, \dsr)\to 0$,
    and since $\overline{W}$ is a metric, $\daleth_k\to\dsr$ in $\overline{W}$.
\end{proof}

\section{Further Discussion and Extensions}\label{sec:further-discussion}
\subsection{Examples of \AbsObjs in Finite-State-Space Environments}

In this section, we include several examples to illustrate the breadth of distributions on the simplex that can be obtained as \absobjs for simple environments.

Figure~\ref{fig:3-state-example} illustrates a kernel density approximation to the \absobj in a three-state MDP, with state-transition kernel given by
\begin{align*}
    \begin{pmatrix}
    0.5 & 0.5 & 0 \\
    0 & 0 & 1 \\
    1/3 & 1/3 & 1.3
    \end{pmatrix} \, ,
\end{align*}
and a discount factor of $\gamma = 0.7$. The figure is specifically created by generating 1,000 trajectories of length 100, which are then converted into visitation distributions, serving as approximate samples of the \absobj, and a kernel density estimator (KDE) is then fitted; we use Seaborn's \texttt{kdeplot} method with default parameters \citep{waskom2021seaborn}.
Also included in the figure are corresponding return distribution estimates, obtained by using the identity in \Eqref{eq:zero-shot-return-distribution} with the generated samples described above, and again using a KDE plot of the resulting return distribution estimator. Observe that since the second state transitions deterministically into the third state, the \absobj for the second state is a scaling and translation of the \absobj of the third state, as predicted by the \absobj Bellman equation in \Eqref{eq:randocc-bellman-equation}.

\begin{figure}[h]
    \vspace{-0.7em}
    \begin{center}
        \begin{tabular}{c c c}
            \includegraphics[keepaspectratio,width=.3\textwidth]{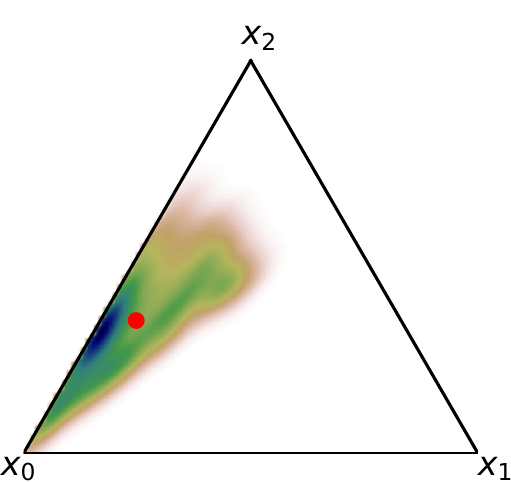} &
            \includegraphics[keepaspectratio,width=.3\textwidth]{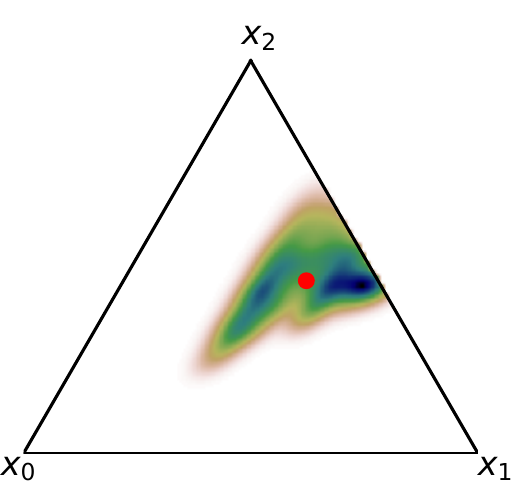} &
            \includegraphics[keepaspectratio,width=.3\textwidth]{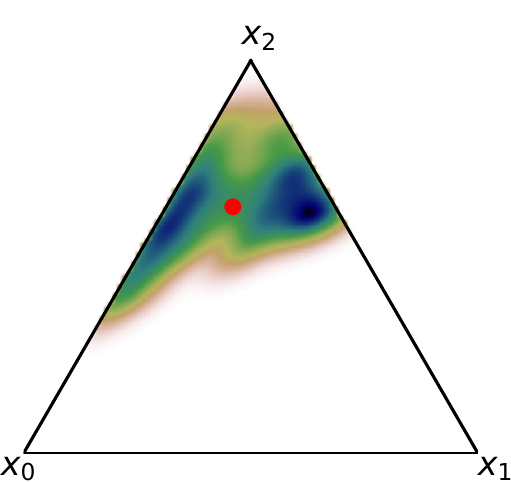}\\
            \hspace{-0.8cm}\includegraphics[keepaspectratio,width=.3\textwidth]{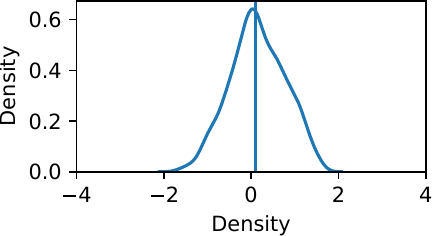} &
            \hspace{-0.8cm}\includegraphics[keepaspectratio,width=.3\textwidth]{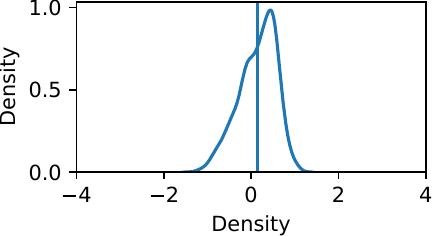} &
            \hspace{-0.8cm}\includegraphics[keepaspectratio,width=.3\textwidth]{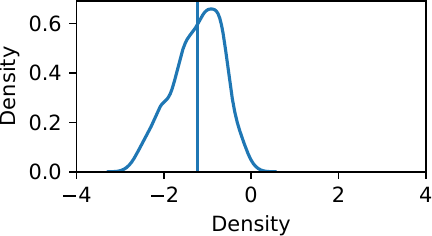}
        \end{tabular}
    \end{center}

    \caption{\textbf{Top:} Kernel density estimate of \absobj. Red dot represents the standard SR. \textbf{Bottom:} Kernel density estimates of return distributions, obtained via \absobj. Vertical lines represent expected return, obtained from standard SR.}
    \label{fig:3-state-example}
    
\end{figure}

\begin{figure}[h]
    \centering
    \null
    \hfill
    \includegraphics[keepaspectratio,width=.3\textwidth]{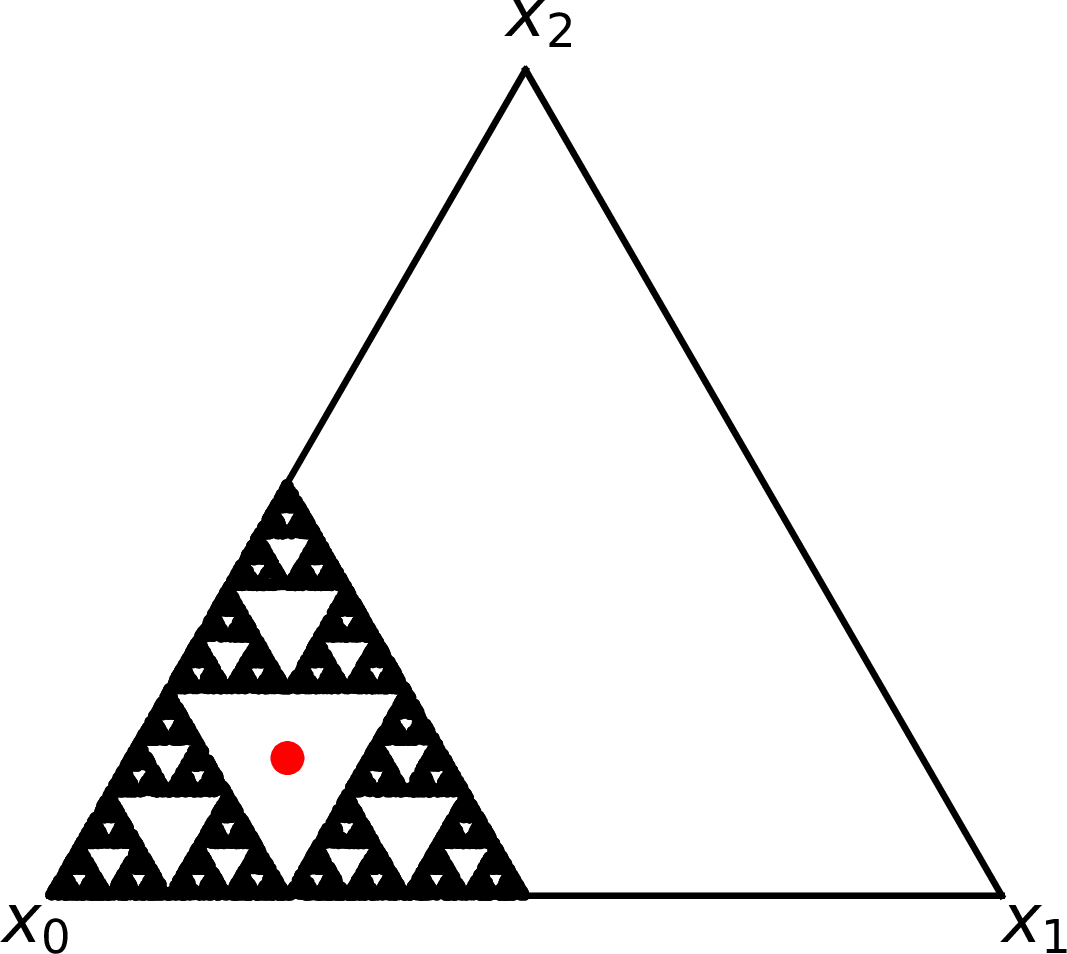}
    \hfill
    \includegraphics[keepaspectratio,width=.3\textwidth]{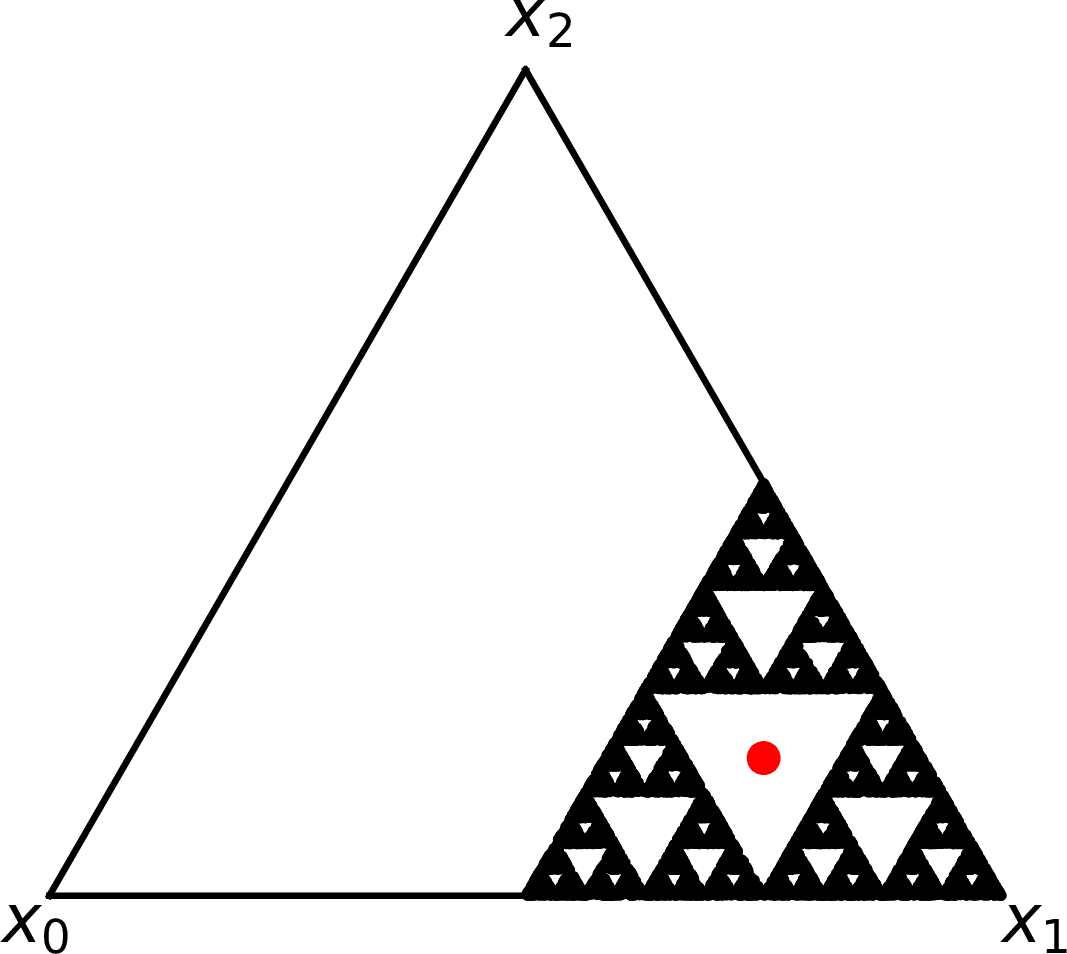}
    \hfill
    \includegraphics[keepaspectratio,width=.3\textwidth]{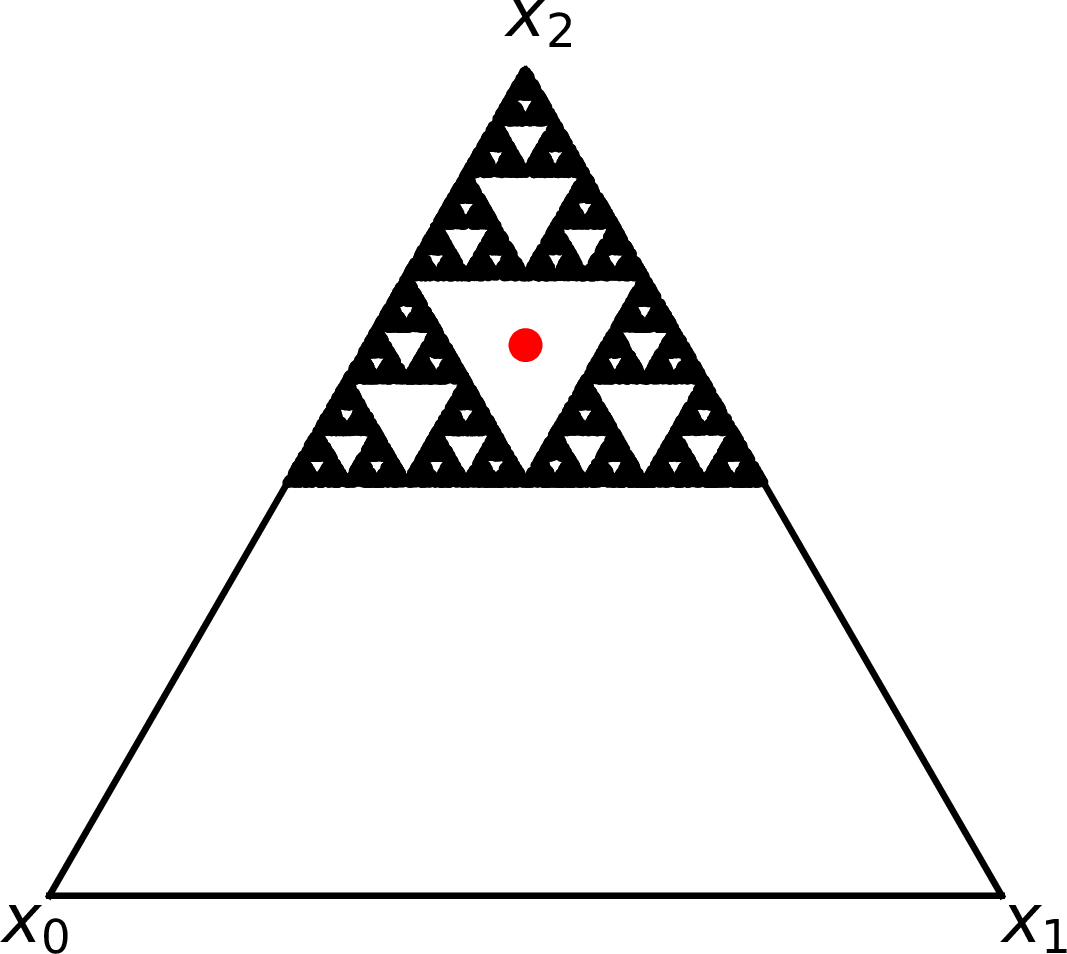}
    \hfill
    \null
    \caption{Monte Carlo estimation of the \absobj at states $x_0$, $x_1$, and $x_2$, in a three-state MDP. Each distribution is supported on a copy of the fractal Sierpi\'nski triangle. Red dot represents the standard SR.}
    \label{fig:gasket}
\end{figure}

In Figure~\ref{fig:gasket}, we plot a Monte Carlo approximation to the \absobj in a three-state environment in which there is an equal probability of jumping to each state in every transition, and the discount factor is $\gamma = 0.5$. The distributions over the simplex in this case are instances of the Sierpi\'nski triangle, a fractal distribution that is neither discrete nor absolutely continuous with respect to Lebesgue measure on the simplex. This can be viewed as a higher-dimensional analogue of the Cantor distribution described in the context of distributional reinforcement learning in \citet[Example~2.11]{bdr2022}.
These plots were generated using 10,000 samples per state, with an episode length of 100.

\subsection{Stochastic reward functions}

In the main paper, we make a running assumption that the rewards encountered at each state are given by a deterministic assumption. In full generality, Markov decision processes allow for the state-conditioned reward to follow a non-trivial probability distribution. In this section, we briefly describe the main issue with extending our approach to dealing with stochastic rewards.

The issue stems from the fact that the mapping from sequences of state $(X_k)_{k \geq 0}$ to the corresponding occupancy distribution $\sum_{k=0}^\infty \gamma^k \delta_{X_k}$ is often not injective. To see why, consider an environment with four states $x_0, x_1, x_2, x_3$ (including a terminal state $x_3$, which always transitions to itself). Consider two state sequences:
\begin{align*}
    &(x_0, x_1, x_2, x_2, x_3, x_3, \ldots) \, , \\
    &(x_0, x_2, x_1, x_1, x_3, x_3, \ldots) \, .
\end{align*}
These sequences give rise to the visitation distributions
\begin{align*}
    &(1-\gamma) \delta_{x_0} + (1-\gamma) \gamma \delta_{x_1} + (1-\gamma) (\gamma^2 + \gamma^3) \delta_{x_2} + \gamma^4 \delta_{x_3} \, , \\
    & (1-\gamma) \delta_{x_0} + (1-\gamma) (\gamma^2 + \gamma^3) \delta_{x_1} + (1-\gamma) \gamma \delta_{x_2} + \gamma^4 \delta_{x_3} \, .
\end{align*}
Now suppose $\gamma = \gamma^2 + \gamma^3$; clearly there is a value of $\gamma \in (0,1)$ satisfying this equation. But for this value of $\gamma$, the two visitation distributions above are identical. In the case of deterministic state-conditioned rewards, the two corresponding returns are also identical in this case. However, in the case of non-deterministic returns, the corresponding distributions over return are distinct. To give a concrete case, consider the setting in which all rewards are deterministically 0, except at state $x_1$, where they are given by the N$(0,1)$ distribution. Then under the first visitation distribution, the corresponding return distribution is the distribution of $\gamma Z$ (where $Z \sim \text{N}(0, 1)$), which has distribution N$(0, \gamma^2)$. In contrast, the return distribution for the second visitation distribution is the distribution of $\gamma^2 Z + \gamma^3 Z'$ (where $Z, Z' \overset{\mathrm{i.i.d.}}{\sim} N(0, 1)$), which has distribution $N(0, \gamma^4 + \gamma^6)$. However, $\gamma^2 \not= \gamma^4 + \gamma^6$, and hence these distributions are not equal.

These observations mean that the framework \emph{can} be extended to handle stochastic rewards in cycle-less environments; that is, environments where each state can be visited at most once in a given trajectory. This incorporates the important class of finite-horizon environments.

\subsection{The successor measure as a linear operator}\label{sec:app:linear-operator}

Here, we recall a key notion from \citet{blier2021learning} used in several proofs that follow. 
Successor measures act naturally as linear operators on the space $\boundedfn{\mathcal{X}}$ of
bounded measurable functions, much in the same way as Markov kernels act as linear operators
(see e.g. \citet{le2016brownian}). Particularly, for any $f\in\boundedfn{\mathcal{X}}$, we write
\begin{equation}\label{eq:sm:operator}
(\Psi^\pi f)(x) = \int_{\mathcal{X}}f(x')\Psi^\pi(\mathrm{d}x'\mid x) \, ,
\end{equation}
noting that $\Psi^\pi(\cdot\mid x)$ is a (probability) measure for each $x\in\mathcal{X}$. Through
this linear operation, the successor measure transforms reward functions $r:\mathcal{X}\to\R$ to value functions $V^\pi_r$,
\begin{align*}
(1-\gamma)V^\pi_r(x) &= \expectation{\pi}{\sum_{t\geq 0}(1-\gamma)\gamma^tr(X_t)\ \mid X_0=x}\\
&= \expectation{X'\sim\Psi^\pi(\cdot\mid x)}{r(X_t)}\\
&= (\Psi^\pi r)(x) \, .
\end{align*}

\section{Experimental Details}\label{sec:experiments}

In this section, we provide additional details relating to the experiments in the main paper.

\subsection{Baselines}\label{appendix:experiments:baselines}
To avoid confounders in our comparative analysis, our baselines were built with largely the same neural architecture and loss
as the $\delta$-models that we train.

\begin{paragraph}{$\gamma$-Model Ensemble}
An ensemble of $\gamma$ models is structurally equivalent to a $\delta$-model: each member of the ensemble is an equally-weighted model atom. Thus, we train the $\gamma$-model ensemble in the same way as the $\delta$-model, but we substitute the model MMD loss with an sum over state MMD losses corresponding to the model atoms.
\end{paragraph}
\begin{paragraph}{Transition Kernel}
In our experiments involving the learned transition kernel $P^\pi$, we again inherit the architecture from the $\delta$-model.
Note that $P^\pi$ is effectively equivalent to a $\delta$-model trained with $\gamma=0$ and with one model atom. One small
adjustment is necessary: a $\delta$-model with $\gamma=0$ will model the distribution over the source state (that is, each model atom will represent a Dirac at the source state) as defined in \Eqref{eq:randocc-bellman-equation}. To account for this, we
can simply shift indices of target states by one timestep, in order to predict the distribution over next states.

\end{paragraph}

\subsection{Hyperparameters}\label{appendix:experiments:hparams}

Unless otherwise specified the default hyperparameters used for our implementation of \algobj are outlined in Table~\ref{table:hparams}. Certain environment specific hyperparameters can be found in Appendix~\ref{sec:experiments:environments}.

\begin{table}
    \centering
    \caption{Default hyperparameters for \algobj.}
    \label{table:hparams}
    {\renewcommand{\arraystretch}{1.1}
    \begin{tabular}{l l}
    \hline
    Hyperparameter & Value \\\hline
    Generator Network & MLP(3-layers, 256 units, ReLU) \\
    Generator Optimizer & Adam($\beta_1 = 0.9$, $\beta_2 = 0.999$) \\
    Generator Learning Rate & $6.25e-5$ \\
    Discriminator Network & iResMLP(2 layers $\times$ 2 blocks, 256 units, ReLU) \\
    Discriminator Optimizer & Adam($\beta_1 = 0.9$, $\beta_2 = 0.999$) \\
    Discriminator Learning Rate & $6.25e-5$ \\
    Discriminator Feature Dimensionality & $8$ output features \\\hline
    Model Kernel & InverseMultiQuadric \\
    Adaptive Model Kernel (Median Heuristic) & True \\
    State Kernel & RationalQuadricKernel$\left(\mathcal{A} = \{ 0.2, 0.5, 1.0, 2.0, 5.0 \}\right)$ \\
    Adaptive State Kernel (Adversarial Kernel) & True \\\hline
    Horizon ($n$-step) & $5$ \\
    Discount Factor ($\gamma$) & $0.95$ \\
    Batch Size & $32$ \\
    Number of State Samples & $32$ \\
    Number of Model Samples & $51$ \\
    Target Parameter Step Size ($\lambda$) & $0.01$ \\
    Noise Distribution & $\omega \in \mathbb{R}^{8} \sim \mathcal{N}(0, I)$ \\
    Number of Gradient Updates & $3e6$\\\hline
    \end{tabular}
    }
\end{table}

\subsection{Environment Details}\label{sec:experiments:environments}

Below we provide specifics of the environments utilized for the experimental results in the paper.

\subsubsection{Windy Gridworld}\label{sec:experiments:windy}
When training a \algobj for the Windy Gridworld experiments, we use 4 model atoms and train for
1 million gradient steps.

Our experiments in Section \ref{sec:deep-dsr} involve two reward functions, namely \texttt{Hopscotch}
and \texttt{Lopsided Checkerboard}. These reward functions have constant rewards in each quadrant,
as shown in Figure \ref{fig:windy:rewards}.

\begin{figure}[h]
    \centering
    \begin{minipage}{0.49\linewidth}
    \centering
    \texttt{Lopsided Checkerboard}\\
    \begin{tabular}{|c|c|}
    \hline
    15 & -10\\\hline
    -2 & 2\\\hline
    \end{tabular}
    \end{minipage}
    \begin{minipage}{0.49\linewidth}
    \centering
    \texttt{Hopscotch}\\
    \begin{tabular}{|c|c|}
    \hline
    3 & -1\\\hline
    -2 & 2\\\hline
    \end{tabular}
    \end{minipage}
    \caption{Reward functions for Windy Gridworld.}
    \label{fig:windy:rewards}
\end{figure}

Moreover, we provide some additional visualizations on predicted return distributions from our
\absobj implementation in Figure \ref{fig:windy:retdist}.

\begin{figure}
\centering
\includegraphics[width=0.7\linewidth]{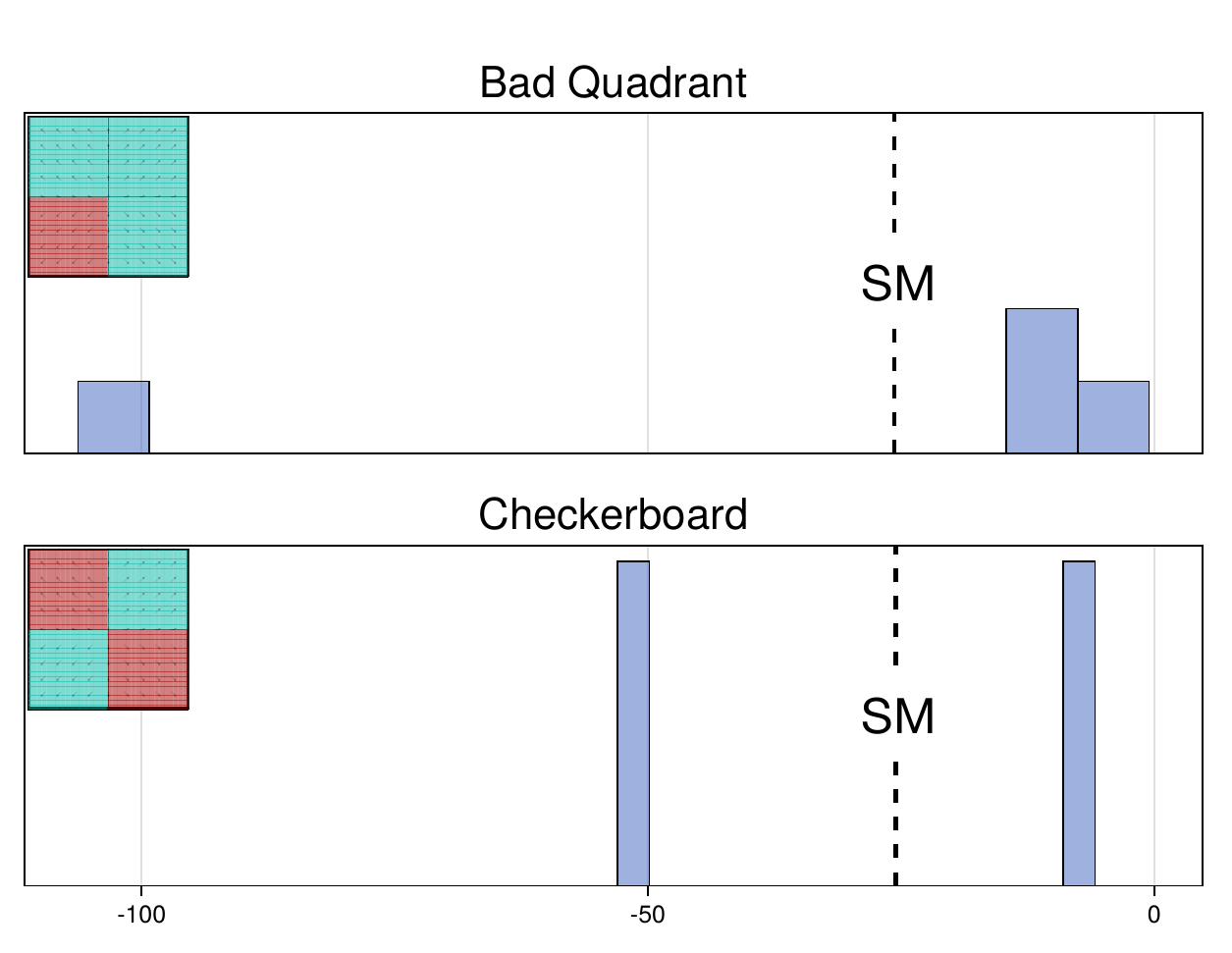}
\caption{Return distribution predictions in Windy Gridworld under the uniform random policy where the source state is the origin. Each row represents a separate reward function depicted by the inset grids, with
red regions denoting negative reward.}\label{fig:windy:retdist}
\end{figure}

Notably, Figure \ref{fig:windy:retdist} demonstrates that the \absobj correctly predict the
fraction of futures that enter the red region, which demonstrates that the \absobj is can
detect when a policy will be likely to violate novel constraints.
\subsubsection{Pendulum}\label{sec:experiments:pendulum}
When training a \algobj for the Pendulum experiments, we use 51 model atoms and train for 3 million
gradient steps.

Our experiments on the Pendulum environment involve zero-shot policy evaluation for rewards that are held out during training.
We considered four reward functions, namely \texttt{Default}, \texttt{Above Horizon}, \texttt{Stay Left}, and \texttt{Counterclockwise Penalty},
which we describe below.

All reward functions are defined in terms of the pendulum angle $\theta\in[-\pi,\pi]$, its angular velocity $\dot\theta$, and the action $a\in\R$. The reward functions are given by

\begin{align*}
    r_{\texttt{Default}}(\theta, \dot\theta, a) &= -\left(\theta^2 + 0.1\dot\theta^2 + 0.001a^2\right)\\
    r_{\texttt{Above Horizon}}(\theta, \dot\theta, a) &= -(\indicator{\theta\geq\pi/2} + 0.1a^2)\\
    r_{\texttt{Stay Left}}(\theta, \dot\theta, a) &= \min(0,\sin\theta)\\
    r_{\texttt{CCW Penalty}}(\theta, \dot\theta, a) &= \indicator{\dot\theta < 0}
\end{align*}

These reward functions (aside from \texttt{Default}) were chosen to model potential constraints
that can be imposed on the system after a learning phase. The \texttt{Above Horizon} reward
imposes extra penalty whenever the pendulum is below the horizon, which may model the presence of an obstacle under the horizon. The \texttt{Stay Left} reward
reinforces the system when the pendulum points further to the left, which could, for instance,
indicate a different desired target for the pendulum. Finally, the \texttt{Counterclockwise Penalty}
($r_{\texttt{CCW Penalty}}$ above) reinforces the system for rotating clockwise, which can model
a constraint on the motor.

\section{Ablation Experiments}
\begin{figure}
\includegraphics[width=0.38\textwidth]{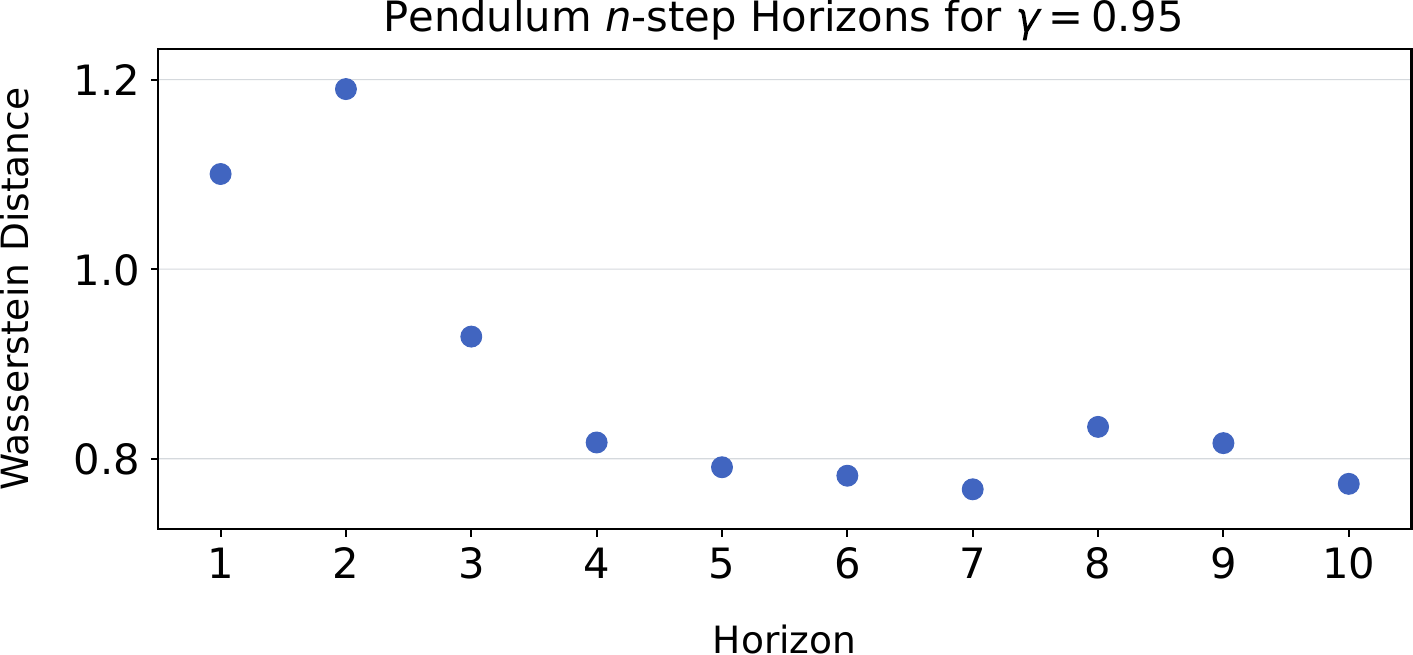}
\includegraphics[width=0.38\textwidth]{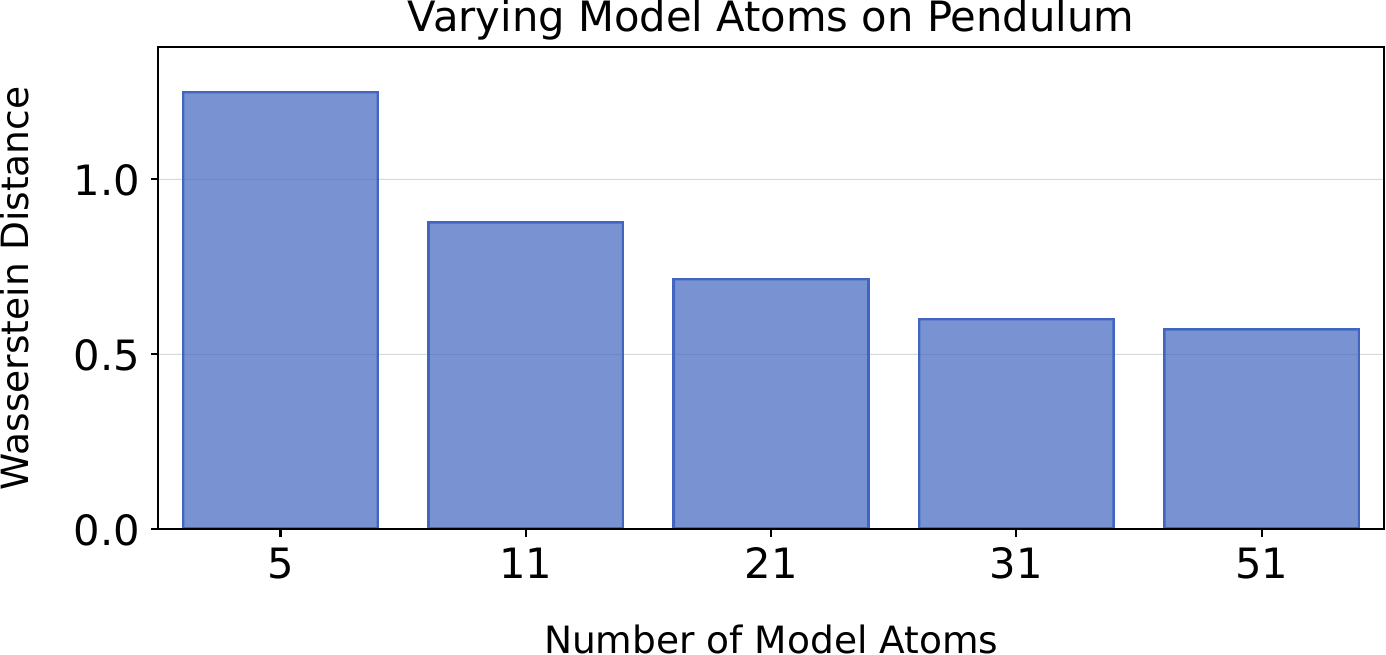}
\includegraphics[width=0.22\textwidth]{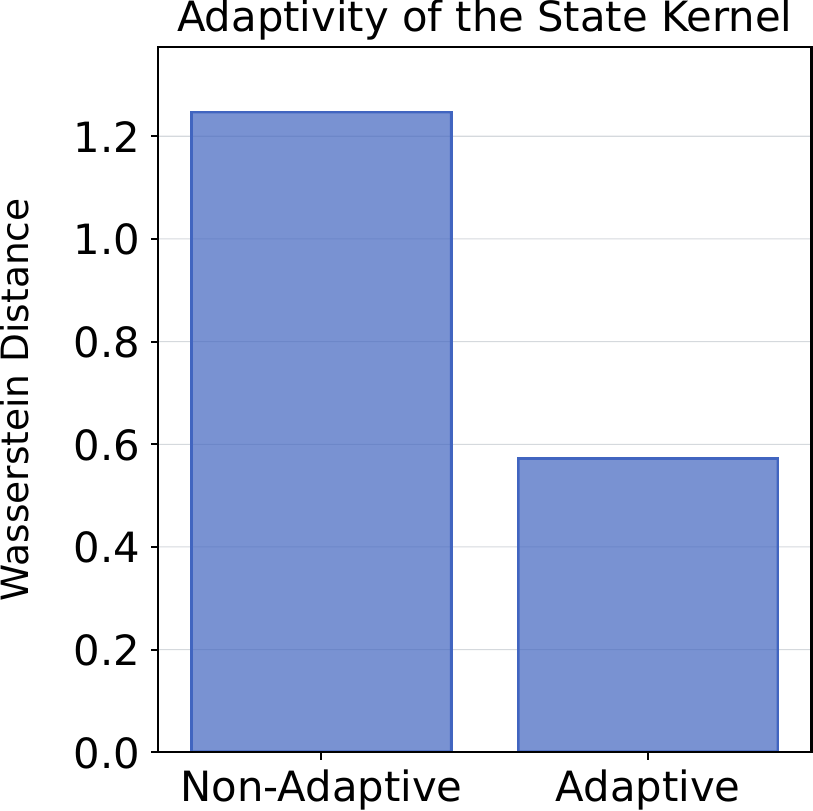}
\caption{Each subplot indicates the Wasserstein distance of the DSM versus the empirical MC return distributions on $9$ source states for $4$ reward functions on Pendulum. \textbf{Left:} Varying the value of $n$ in the \algobj multi-step bootstrapped target. \textbf{Middle:} varying the number of model atoms. \textbf{Right:} Selectively applying an adaptive (adversarial) or non-adaptive state kernel.}
\label{fig:ablations}
\end{figure}

As mentioned in Section~\ref{sec:deep-dsr} there are several practical considerations when learning \algobjs. We further expand on three crucial details: $n$-step bootstrapping, adaptive kernels, and how the number of model atoms affects our approximation error.

\textbf{$n$-step bootstrapping.}
In order to allow \algobjs to learn longer horizons than typically possible with $\gamma$-models \citep{janner2020gamma} we employ the use of $n$-step bootstrapping when constructing our target distribution.
The choice of $n$ is critical; if $n$ is too small, the update is over-reliant on bootstrapping, leading to instability. Conversely, for large $n$ it becomes impractical to store long sequences.

For these reasons it is worth understanding how \algobjs interacts with the value of $n$, which can help guide the selection of this parameter for any type of geometric horizon model \citep[e.g.,][]{janner2020gamma, thakoor2022generalised}. To this end, we perform a sweep over $n \in \{1, 2, \dots, 10\}$ to understand how this affects the Wasserstein distance between the DSM-estimated return compared to the empirical MC return distribution. We train a \algobj with $\gamma = 0.95$ for each $n$ on the Pendulum environment with all other hyperparameters remaining fixed as in Appendix~\ref{appendix:experiments:hparams}. Figure~\ref{fig:ablations} (left) shows the Wasserstein distance averaged over $9$ source states for the four reward functions outlined in Appendix~\ref{sec:experiments:pendulum}.

We can see that $n$-step bootstrapping does indeed help us to learn better approximations until around $n = 5$ where the benefits are less clear. This corresponding to the bootstrapped term accounting for $\approx 80\%$ of the probability mass of the target distribution.

\textbf{Adversarial kernel.}
Given the non-stationary of our target distributions, we found it crucial to employ an adaptive kernel in the form of an adversarial kernel \citep{mmdgan} for the state kernel. Note that the model kernel is itself a function of the state kernel so by learning an adversarial state kernel we are able to adapt our model kernel as well.

To validate our decision to employ an adversarial kernel we train a \algobj with and without an adaptive kernel. The adaptive kernel is the one described in Appendix~\ref{appendix:experiments:hparams}.
The non-adaptive kernel omits the application of the learned embedding network, that is, the kernel is a mixture of rational quadric kernels,
$$
\rho(d) = \sum_{\alpha \in \mathcal{A}} \left( 1 + \frac{d}{2\alpha} \right)^{-\alpha} \, ,
$$
for $\mathcal{A} = \{ 0.2, 0.5, 1.0, 2.0, 5.0 \}$ as per \citet{binkowski18demystifying}. Figure~\ref{fig:ablations} (right) shows that the Wasserstein distance is nearly halved when applying the adversarial kernel.

\textbf{Number of model atoms.}
As the number of model atoms increases we expect to better approximate $\dsr$. To get an idea of how our approximation is improving as we scale the number of atoms we compare the Wasserstein distance from our \algobj to the empirical MC return distributions for $\{5, 11, 21, 31, 41, 51 \}$ model atoms. These results are presented in Figure~\ref{fig:ablations} (Middle). As expected we obtain a better approximation to $\dsr$ when increasing the number of model atoms. Further scaling the number of model atoms should continue to improve performance at the cost of compute. This is a desirable property for risk-sensitive applications where better approximations are required.
\label{appendix:experiments:ablation}

\section{Additional Results}\label{sec:additional_results}
\begin{restatable}{proposition}{propSRDeterminesDSR}\label{prop:SRDeterminesDSR}
    The \absobj is determined by the standard SR. In other words, given $\Psi^\pi$, one can mathematically derive $\daleth^\pi$.
\end{restatable}

\begin{proof}
    To establish Proposition \ref{prop:SRDeterminesDSR}, it suffices to show that the one-step transition kernel
    $\statetransitionkernel$ for a given policy $\pi$ can be recovered exactly from $\Psi^\pi$. This is because
    $\statetransitionkernel$ contains all possible structural information about the environment and the policy's
    dynamics, so it contains all information necessary to construct the \absobj. When $\mathcal{X}$ is finite,
    Lemma \ref{lem:SRDeterminesKernel:tabular} shows that $\Psi^\pi$ encode $\statetransitionkernel$, and Lemma
    \ref{lem:SRDeterminesKernel:continuous} demonstrates this for the more general class of state space $\mathcal{X}$ considered in this paper.
\end{proof}

\begin{lemma}\label{lem:SRDeterminesKernel:tabular}
Let $\mathcal{X}$ be finite, and let $\Psi^\pi$ denote the successor representation for a given policy $\pi$.
Then $\statetransitionkernel$ can be recovered exactly from $\Psi^\pi$.
\end{lemma}

\begin{proof}
    Consider a policy $\pi : \mathcal{X} \rightarrow \mathscr{P}(\mathcal{A})$ with discounted visitation distributions $\Psi^\pi$. We consider the state transition matrix $P^\pi\in\R^{|\mathcal{X}|\times|\mathcal{X}|}$ where $P^\pi_{x,x'} = \statetransitionkernel(x'\mid x)$. Recall that $\Psi^\pi = (1-\gamma)(I - \gamma P^\pi)^{-1}$, so rearranging we have $P^\pi = \gamma^{-1}(I - (1-\gamma)(\Psi^\pi)^{-1})$. Therefore the one-step state-to-state transition probabilities are determined by $\Psi^\pi$, and since $\daleth^\pi$ is a function of the one-step transition probabilities, the conclusion follows.
\end{proof}

\begin{lemma}\label{lem:SRDeterminesKernel:continuous}
Let $\mathcal{X}$ be a complete, separable metric space endowed with its Borel $\sigma$-field $\Sigma$, and let $\Psi^\pi$ denote the successor measure for a given policy
$\pi$. Then $\Psi^\pi$ encodes $\statetransitionkernel$, in the sense that $\statetransitionkernel$ can be
expressed as a function of $\Psi^\pi$ alone.
\end{lemma}
\begin{proof}
Recall the definition of the successor measure $\Psi^\pi:\Sigma\to\R^+$,
\begin{align*}
    \Psi^\pi(A\mid x) &= (1-\gamma)\sum_{t\geq 0}\gamma^t\Pr(X_t\in A\mid X_0=x) \, .
\end{align*}
As shown above in Appendix \ref{sec:app:linear-operator}, $\Psi^\pi$ acts as a linear operator on $\boundedfn{\mathcal{X}}$ according to $(\Psi^\pi f)(x) = \expectation{X'\sim\Psi^\pi(\cdot\mid x)}{f(X')}$.
We denote by $P^\pi:\boundedfn{\mathcal{X}}\to\boundedfn{\mathcal{X}}$ the Markov kernel corresponding
to $\statetransitionkernel$, where $\boundedfn{\mathcal{X}}$ denotes the space of bounded and measurable
functions on $\mathcal{X}$. The operator $P^\pi$ acts on a function $f\in\boundedfn{\mathcal{X}}$ according to
\begin{align*}
    (P^\pi f)(x) &= \int_{\mathcal{X}}f(x')\statetransitionkernel(\mathrm{d}x'\mid x) = \expectation{X'\sim \statetransitionkernel(\cdot\mid x)}{f(X')} \, .
\end{align*}
That is, $(P^\pi f)(x)$ computes the expected value of $f$ over the distribution of next states, conditioned
on a starting state. Returning to the definition of the successor measure, for any $f\in\boundedfn{\mathcal{X}}$, we have
\begin{align*}
    (\Psi^\pi f)(x)
    &= \int_{\mathcal{X}}f(x')\Psi^\pi(\mathrm{d}x'\mid x)\\
    &= (1-\gamma)\int_{\mathcal{X}}f(x')\sum_{t\geq 0}\gamma^t(\statetransitionkernel)^t(\mathrm{d}x'\mid x)\\
    &= (1-\gamma)\sum_{t\geq 0}\gamma^t\int_{\mathcal{X}}f(x')(\statetransitionkernel)^t(\mathrm{d}x'\mid x)\\
    &= (1-\gamma)\sum_{t\geq 0}\gamma^t((P^\pi)^tf)(x)
\end{align*}
where the third step invokes Fubini's theorem, given the boundedness of $f$ and $\statetransitionkernel$. We have shown that
\begin{align*}
    \Psi^\pi = (1-\gamma)\sum_{t\geq 0}\gamma^t(P^\pi)^t \, ,
\end{align*}
where the correspondence is with respect to the interpretation of $\Psi^\pi$ as a linear operator on $\boundedfn{\mathcal{X}}$. \citet[Theorem 2]{blier2021learning} show that $\sum_{t\geq 0}\gamma^t(P^\pi)^t = (\identity - \gamma P^\pi)^{-1}$ as linear operators on $\boundedfn{\mathcal{X}}$, where $\identity$ is the identity map on
$\boundedfn{\mathcal{X}}$. As a consequence, $\Psi^\pi$ is proportional to the inverse of a linear operator, so it is itself an invertible linear operator, where
\begin{align*}
    (\Psi^\pi)^{-1} &= \frac{1}{1-\gamma}(\identity - \gamma P^\pi)
\end{align*}
and hence
\begin{align*}
    P^\pi &= \gamma^{-1}\left(\identity - (1-\gamma)(\Psi^\pi)^{-1}\right) \, .
\end{align*}
Again, the correspondence is established for $P^\pi$ as a linear operator on $\boundedfn{\mathcal{X}}$. However, we can now recover
the measures $\statetransitionkernel(\cdot\mid x)$ according to
\begin{align*}
    \statetransitionkernel(A\mid x) &= \int_A\statetransitionkernel(\mathrm{d}x'\mid x)\\
    &= \int_{\mathcal{X}}\characteristic{A}(x')\statetransitionkernel(\mathrm{d}x'\mid x) && (\characteristic{A}(y)\triangleq\indicator{y\in A})\\
    &= (P^\pi\characteristic{A})(x)\\
    &= \left(\gamma^{-1}(\identity - (1-\gamma)(\Psi^\pi)^{-1})\characteristic{A}\right)(x)
\end{align*}
where $\characteristic{A}\in\boundedfn{\mathcal{X}}$ for any measurable set $A$. Thus, we have shown that $\statetransitionkernel(\cdot\mid x)$
can be reconstructed from $\Psi^\pi$ alone, as claimed.
\end{proof}

\end{document}